\documentclass{article}

\usepackage{bm}
\usepackage{url,multirow,comment}
\usepackage{amsmath,amssymb,amsthm}
\usepackage{algorithm}
\usepackage[hmarginratio=1:1,top=20mm,columnsep=20pt,a4paper, total={6.8in, 10in}]{geometry}
\usepackage{graphicx}
\usepackage{subcaption}
\usepackage[compatible]{algpseudocode}
\usepackage[flushleft]{threeparttable}
\usepackage{graphicx,color}
\usepackage[figuresright]{rotating}
\usepackage[utf8]{inputenc} 
\usepackage[T1]{fontenc}    
\usepackage{booktabs}       
\usepackage{enumitem}
\usepackage{amsfonts}       
\usepackage{amssymb,mathtools}
\usepackage{nicefrac}       
\usepackage{microtype}      %
\usepackage{xcolor}
\usepackage{pbox}
\usepackage{todonotes}
\usepackage[bookmarks=false,colorlinks=true,urlcolor=black,citecolor=black,linkcolor=black]{hyperref}
\usepackage{cleveref}
\usepackage{bm}
\usepackage{enumitem}
\usepackage{bbm}
\usepackage{cite}

\newtheorem{lem}{Lemma}
\newtheorem{thm}{Theorem}
\newtheorem{defn}{Definition}

\newtheorem{eg}{Example}

\newtheorem{propty}{Property}
\newtheorem{rem}{Remark}

\crefname{align}{}{}
\Crefname{align}{}{}
\crefname{thm}{theorem}{theorems}
\Crefname{thm}{Theorem}{Theorems}
\crefname{eg}{example}{examples}
\Crefname{eg}{Example}{Examples}
\crefname{clm}{claim}{claims}
\Crefname{clm}{Claim}{Claims}
\Crefname{coro}{Corollary}{Corollaries}
\Crefname{lem}{Lemma}{Lemmas}
\Crefname{sec}{Section}{Sections}
\crefname{app}{appendix}{appendices}
\Crefname{app}{Appendix}{Appendices}
\Crefname{part}{Part}{Parts}
\crefname{prop}{proposition}{propositions}
\Crefname{prop}{Proposition}{Propositions}
\Crefname{propty}{Property}{Properties}
\crefname{figure}{fig.}{figures}
\Crefname{figure}{Fig.}{Figures}
\crefname{defn}{definition}{definitions}
\Crefname{defn}{Definition}{Definitions}
\crefname{fact}{fact}{facts}
\Crefname{fact}{Fact}{Facts}
\crefname{appendix}{appendix}{appendices}
\Crefname{appendix}{Appendix}{Appendices}
\crefname{algo}{algorithm}{algorithms}
\Crefname{algo}{Algorithm}{Algorithms}
\crefname{algorithm}{algorithm}{algorithms}
\Crefname{algorithm}{Algorithm}{Algorithms}
\crefname{assm}{assumption}{assumptions}
\Crefname{assm}{Assumption}{Assumptions}
\crefname{conj}{conjecture}{conjectures}
\Crefname{conj}{Conjecture}{Conjectures}
\crefname{obs}{observation}{observations}
\Crefname{obs}{Observation}{Observations}
\crefname{rem}{remark}{remarks}
\Crefname{rem}{Remark}{Remarks}

\begin{document}

\title{Is There a Trade-Off Between Fairness and Accuracy? \\A Perspective Using Mismatched Hypothesis Testing}
\author{Sanghamitra Dutta,$^{1,2}$ Dennis Wei,$^1$ Hazar Yueksel,$^1$\\ Pin-Yu Chen,$^1$ Sijia Liu,$^1$ and Kush R. Varshney$^1$\\
\thanks{Author Contacts: S. Dutta (sanghamd@andrew.cmu.edu), D. Wei (dwei@us.ibm.com), H. Yueksel (hazar.yueksel@ibm.com), P.-Y. Chen (pin-yu.chen@ibm.com), S. Liu (sijia.liu@ibm.com) and K. R. Varshney (krvarshn@us.ibm.com).}
\thanks{This work was done when S. Dutta was a research intern at IBM Research.}
\thanks{This paper appears in the Proceedings of the 37th International Conference on Machine Learning, pp. 2803--2813, 2020.}
\normalsize $^1$IBM Research, $^2$Carnegie Mellon University   
}
\date{}

\maketitle

\begin{abstract}
A trade-off between accuracy and fairness is almost taken as a given in the existing literature on fairness in machine learning. Yet, it is not preordained that accuracy should decrease with increased fairness. Novel to this work, we examine fair classification through the lens of \emph{mismatched hypothesis testing}: trying to find a classifier that distinguishes between two ideal distributions when given two mismatched distributions that are biased. Using Chernoff information, a tool in information theory, we theoretically demonstrate that, contrary to popular belief, there always exist ideal distributions such that optimal fairness and accuracy (with respect to the ideal distributions) are achieved simultaneously: there is no trade-off. Moreover, the same classifier yields the lack of a trade-off with respect to ideal distributions while yielding a trade-off when accuracy is measured with respect to the given (possibly biased) dataset. To complement our main result, we formulate an optimization to find ideal distributions and derive fundamental limits to explain why a trade-off exists on the given biased dataset. We also derive conditions under which active data collection can alleviate the fairness-accuracy trade-off in the real world. Our results lead us to contend that it is problematic to measure accuracy with respect to data that reflects bias, and instead, we should be considering accuracy with respect to ideal, unbiased data.
\end{abstract}

\section{Introduction}
\label{sec:introduction}

This work addresses a fundamental question in the field of algorithmic fairness~\cite{calmon2017optimized,dwork2012fairness,agarwal2018reductions,hardt2016equality,ghassami2018fairness,kusner2017counterfactual,kilbertus2017avoiding,zemel2013learning}:
\begin{center}
\textit{Is there a trade-off between fairness and accuracy?}
\end{center}

The existence of this trade-off has been pointed out in several existing works~\cite{menon2018cost,chen2018my,zhao2019inherent} that also propose different theoretical approaches to characterize it. Yet, it is not preordained as to why such a trade-off should exist between fairness and accuracy. For instance, \cite{friedler2016possibility} and \cite{yeom2018discriminative} suggest that the observed features in a machine learning model (e.g., test scores) are a possibly noisy mapping from features in an abstract construct space (e.g., true ability) where there is no such trade-off. Then, why does correcting for biases worsen predictive accuracy in the real world? We believe there is value in stepping back and reposing the fundamental question.

In this work, our main assertion is that the trade-off between accuracy and fairness (in particular, equal opportunity~\cite{hardt2016equality}) in the real world is due to noisier (and hence biased) mappings for the unprivileged group due to historic differences in opportunity, representation, etc., making their positive and negative labels ``less separable.'' To concretize this idea, we adopt a novel viewpoint on fair classification: the perspective of mismatched hypothesis testing. In mismatched hypothesis testing, the goal is to find a classifier that distinguishes between two ``ideal'' distributions, but instead, one only has access to two mismatched distributions that are biased. Our most important result is to theoretically show that for a fair classifier with sub-optimal accuracy on the given biased data distributions, there always exist ideal distributions such that fairness and accuracy are in accord when accuracy is measured with respect to the ideal distributions. Through this perspective, there is no trade-off between fairness and accuracy.

Our contributions in this work are as follows:

\begin{itemize}[leftmargin=*]
\item \textit{Concept of separability to quantify accuracy-fairness trade-off in the real world:} For a group of people in an observed dataset, we quantify the ``separability'' into positive and negative class labels using Chernoff information, an information-theoretic approximation to the best exponent of the probability of error in binary classification. We demonstrate (in Theorem~\ref{thm:separability}) that if the Chernoff information of one group is lower than that of the other in the observed dataset, then modifying the best classifier using a group fairness criterion compromises the error exponent (representative of accuracy) of one or both the groups, explaining the accuracy-fairness trade-off. Not only do these tools demonstrate the existence of a trade-off (as also demonstrated in some existing works~\cite{menon2018cost,chen2018my} using alternative formulations), but they also enable us to approximately quantify the trade-off, e.g., how close can we bring the probabilities of false negative for two groups in an attempt to attain equal opportunity for a certain compromise on accuracy (see Fig.~\ref{fig:active} in Section~\ref{sec:numerical}). The existence of this trade-off prompts us to contend that accuracy of a classifier with respect to the existing (possibly biased) dataset is a problematic measure of performance. Instead, one should consider accuracy with respect to an ideal dataset that is an unbiased representation of the population.

\item \textit{Ideal distributions where fairness and accuracy are in accord:} Novel to this work, we examine the problem of fair classification through the lens of mismatched hypothesis testing. We show (in Theorem~\ref{thm:feasibility}) that there exist ideal distributions such that both fairness (in the sense of equal opportunity on both the existing and the ideal distributions) and accuracy (with respect to the ideal distributions) are in accord. We also formulate an optimization to show how to go about finding such ideal distributions in practice. The ideal distributions provide a target to shift the given biased distributions toward and to evaluate accuracy on. Their interpretation can be two-fold: (i) plausible distributions in the observed space resulting from an ``unbiased'' mapping from the construct space; or (ii) candidate distributions in the construct space itself (discussed further in Section~\ref{subsec:accord}).

\item \textit{Criterion to alleviate the accuracy-fairness trade-off in the real world:} Next, we also address another important question, i.e., when can we alleviate the accuracy-fairness trade-off in the real world that we must work in, specifically through additional data collection. We derive an information-theoretic criterion (in Theorem~\ref{thm:explainability}) under which collecting more features improves separability, and hence, accuracy in the real world, alleviating the trade-off. This can also inform our choice of the ideal distributions. Our analysis serves as a technical explanation for the success of active fairness~\cite{NoriegaCamperoBGP2019,BakkerNTSVP2019, chen2018my} that uses additional features to improve fairness.

\item \textit{Numerical example:} We demonstrate how the analysis works through a simple numerical example (with analytical closed-forms).

\end{itemize}

\textbf{Related Work:} We note that several existing works, such as \cite{garg2019tracking}, \cite{menon2018cost}, \cite{chen2018my}, and \cite{zhao2019inherent}, have also used information theory or Bayes risk to characterize the accuracy-fairness trade-off. However, computing Bayes risk is not straightforward. Indeed, even for Gaussians, one resorts to Chernoff bounds to approximate the Q-function. Chernoff information is an approximation for Bayes risk that has a tractable geometric interpretation (see Fig.~\ref{fig:trade-off}). This enables us to numerically compute the accuracy-fairness trade-off (Fig.~\ref{fig:active}), and also understand ``how much'' accuracy can be improved by data collection, going beyond the assertion that there is some improvement. To the best of our knowledge, existing works have pointed out the existence of a trade-off based on Bayes risk but have not provided a method to exactly compute it, motivating us to introduce the additional tool of Chernoff information to do so approximately. Furthermore, this work goes beyond characterizing the trade-off imposed by the given dataset. Our novelty lies in adopting the perspective of mismatched detection and demonstrating that there exist ideal distributions such that both fairness and accuracy are in accord when accuracy is measured with respect to the ideal distributions. Other very recent works related to accuracy-fairness trade-offs include~\cite{sabato2020bounding,kim2020model,blum2019recovering}.

The recent works of \cite{wick2019unlocking} and \cite{sharmadata} further elucidate the significance of Theorem~\ref{thm:feasibility} and how it presents an insight that contradicts ``the prevailing wisdom,'' i.e., there exists an ideal dataset for which fairness and accuracy are in accord. In a sense, our work provides a theoretical foundation that complements the empirical results of \cite{wick2019unlocking} and \cite{sharmadata}, clarifying when a trade-off exists and when it does not. 

There are also several existing methods of pre-processing data to generate a fair dataset~\cite{CalmonWVRV2018,feldman2015certifying,zemel2013learning}. Here, our goal is not to propose another competing strategy of fairness through pre-processing. Instead, our focus is to theoretically demonstrate that there exists an ideal dataset such that a fair classifier is also optimal in terms of accuracy, which has not been formally shown before. We also focus on equal opportunity rather than statistical parity (as in \cite{CalmonWVRV2018}).

Our tools share similarities with \cite{varshney2018interpretability} (that demonstrates how explainability can improve Chernoff information), as well as the theory of hypothesis testing in general \cite{lee2012generalized,cover2012elements}. Our contribution lies in using these tools in fair machine learning, where they have not been used to the best of our knowledge (e.g., in the previous analyses of~\cite{menon2018cost,zhao2019inherent,chen2018my}). 

\begin{rem}[Population Setting] In this work, we operate in the population setting (motivated from \cite{gretton2007kernel,ravikumar2009sparse,scott2013classification}), i.e., the limit as the number of samples goes to infinity, allowing use of the probability distributions of the data. This allows us to represent binary classifiers as likelihood ratio detectors (also called Neyman-Pearson (NP) detectors) and quantify the fundamental limits on the accuracy-fairness trade-off. Indeed, given any classifier, there always exists a likelihood ratio detector which is at least as good (see NP Lemma in \cite{cover2012elements}). 
\label{rem:population_setting} 
\end{rem}

\section{Preliminaries}
\label{sec:preliminaries}
\begin{figure}
\centering
\includegraphics[height=3cm]{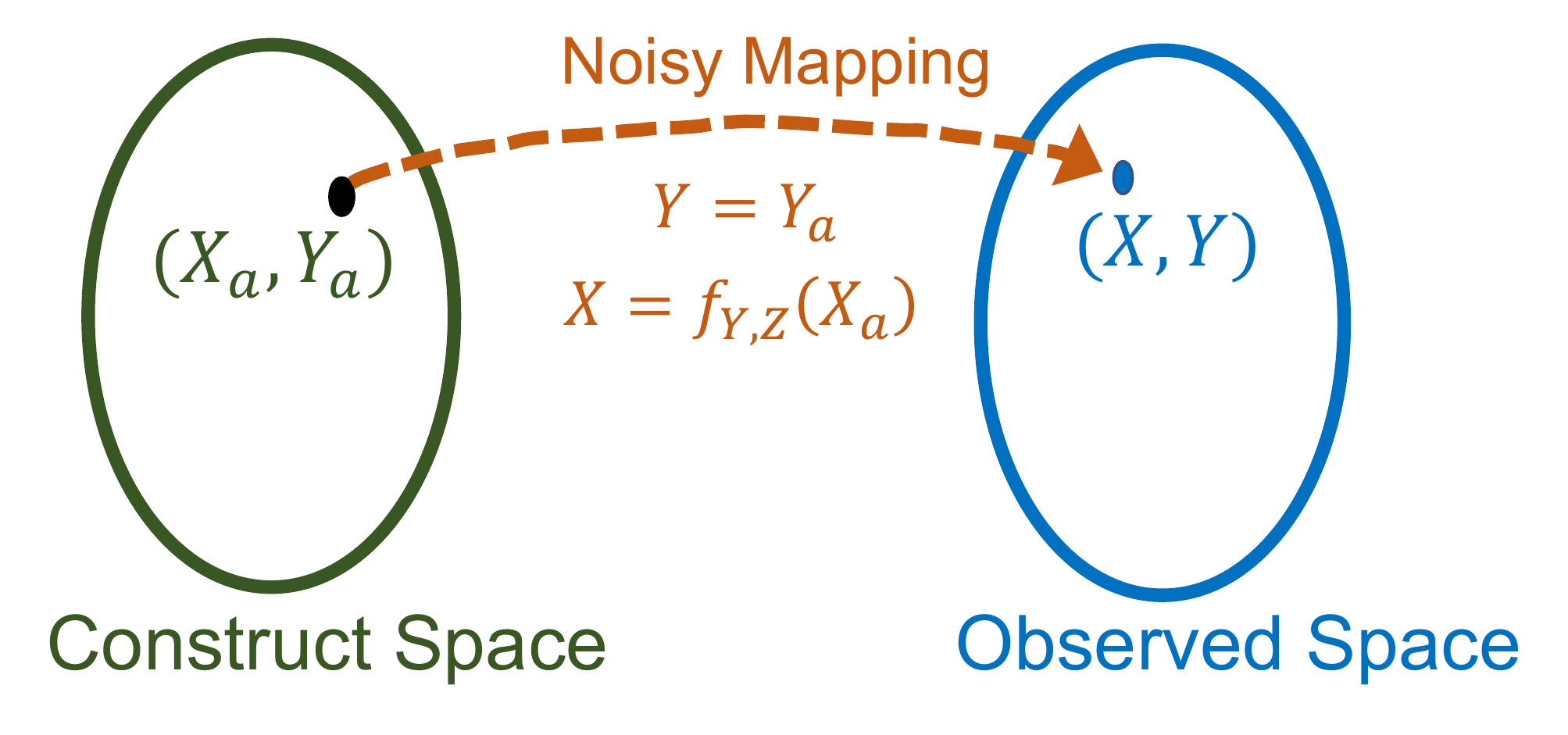}
\caption{Pictorial illustration of the setup. \label{fig:notations}}
\end{figure}
\textbf{Setup:}
In this work, we focus on binary classification, which arises commonly in practice in the fairness literature, e.g., in deciding whether a candidate should be accepted or rejected in applications such as hiring, lending, etc.
We let $Z$ denote the protected attribute, e.g., gender, race, etc. Without loss of generality, let $Z=0$ be the unprivileged group and $Z=1$ be the privileged group.

Inspired by \cite{yeom2018discriminative} and \cite{friedler2016possibility}, we assume that there is an abstract construct space where $X_a$ is the feature (e.g., true ability) and $Y_a$ is the true label (i.e., takes value $0$ or $1$).  The construct space is not directly accessible to us. In the real world, we instead have access to an observed space where $X$ denotes the feature vector and $Y$ denotes the true label (i.e., takes value $0$ or $1$). For the sake of simplicity, we assume $Y_a=Y$ based on \cite{yeom2018discriminative}.\footnote{This is consistent with the ``What You See Is What You Get'' worldview in \cite{yeom2018discriminative} where label bias can be ignored and our chosen measure of fairness, i.e., equal opportunity is justified as a measure of fairness.} The observed features are derived from features in the construct space as follows: $X=f_{Y,Z}(X_a)$ where $f_{Y,Z}(\cdot)$ is a possibly noisy mapping that can depend on $Y$ and $Z$ (also see Fig.~\ref{fig:notations}).

Let the features in the given dataset in the observed space have the following distributions: $X|_{Y=0,Z=0} {\sim} P_0(x)$ and $X|_{Y=1,Z=0} {\sim} P_1(x).$ Similarly, $X|_{Y=0,Z=1} {\sim} Q_0(x)$ and $X|_{Y=1,Z=1} {\sim} Q_1(x).$  
\begin{figure*}
\centering
\includegraphics[height=2.41cm]{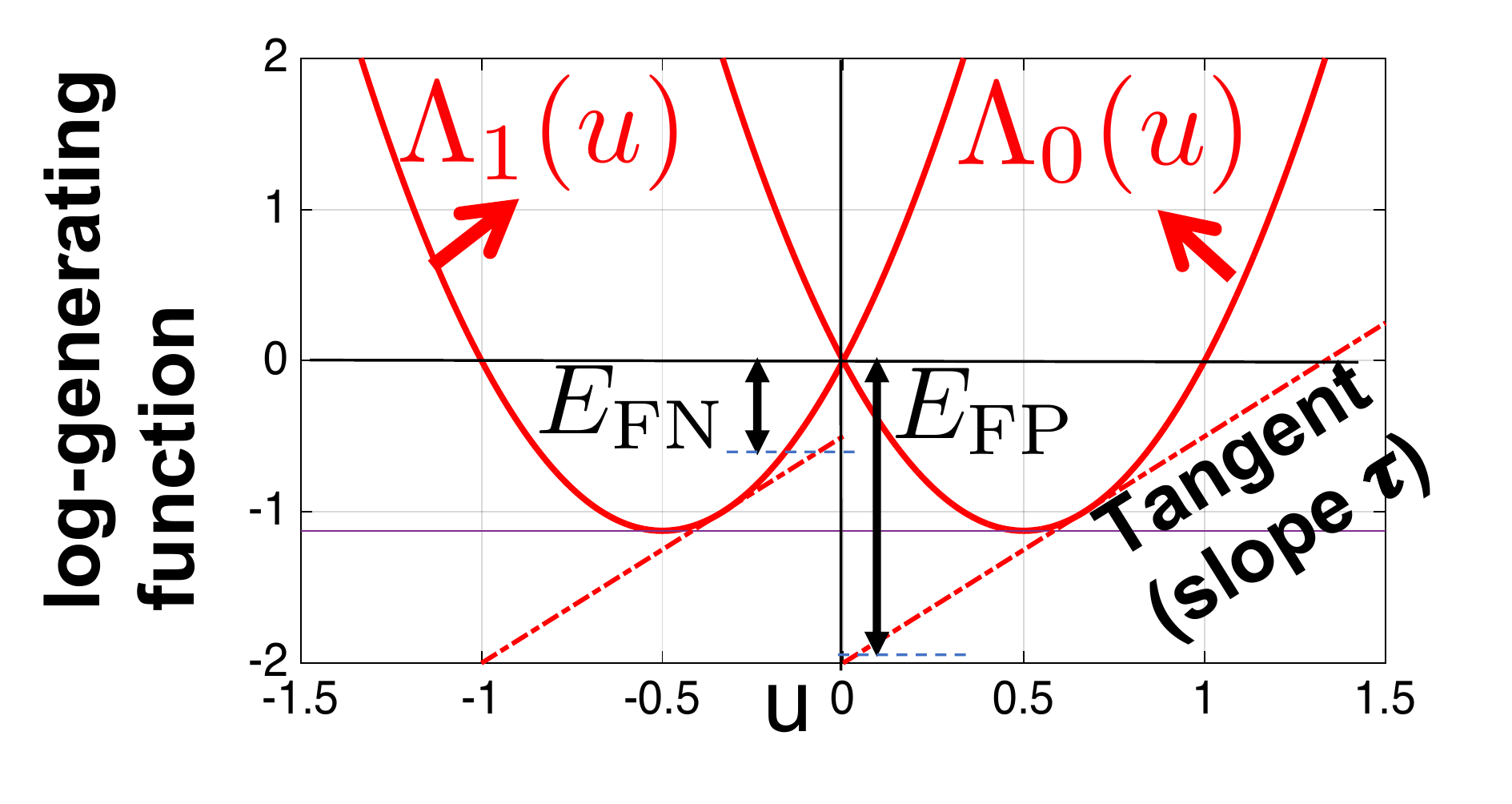}
\includegraphics[height=2.4cm]{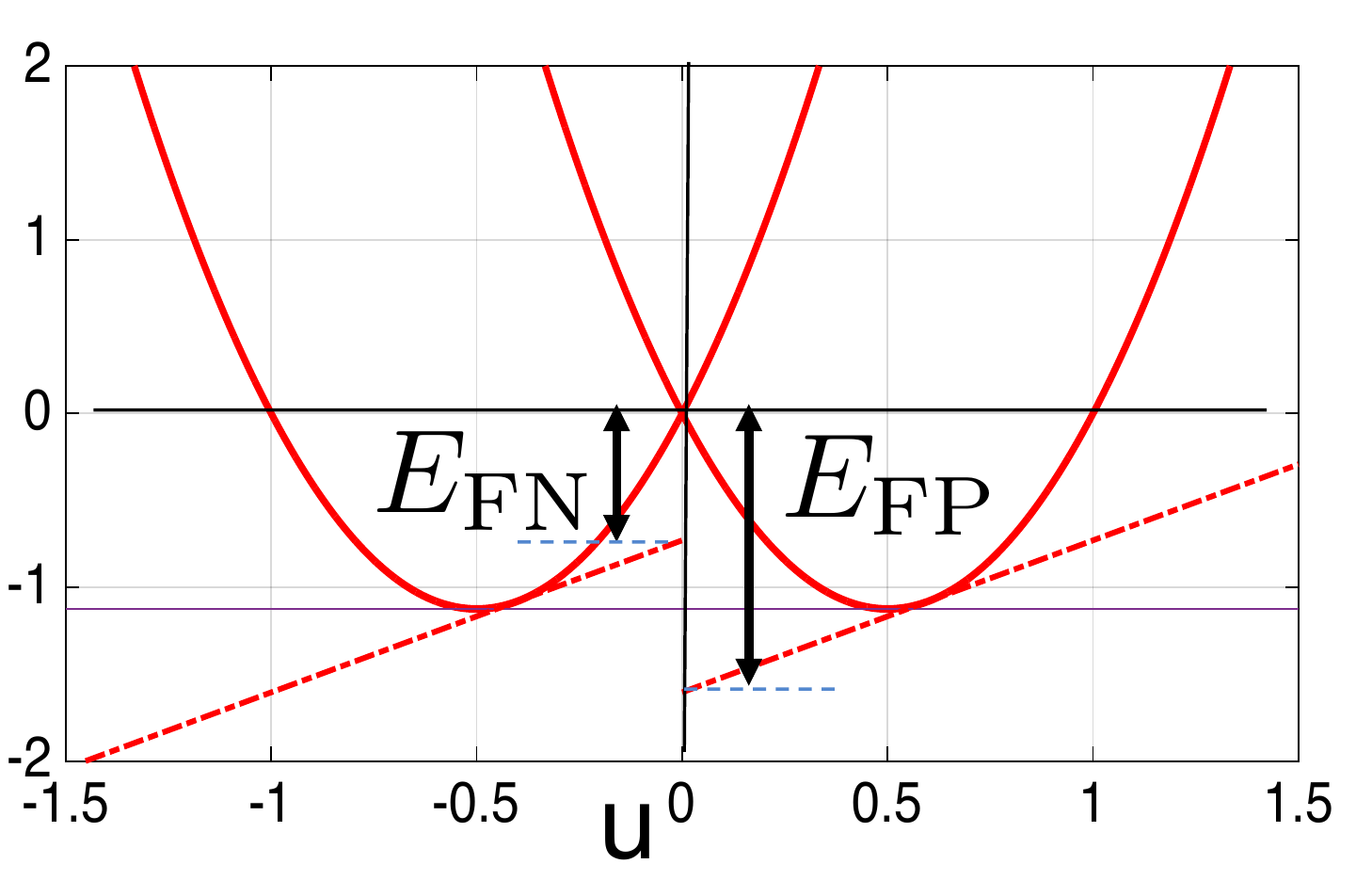}
\includegraphics[height=2.4cm]{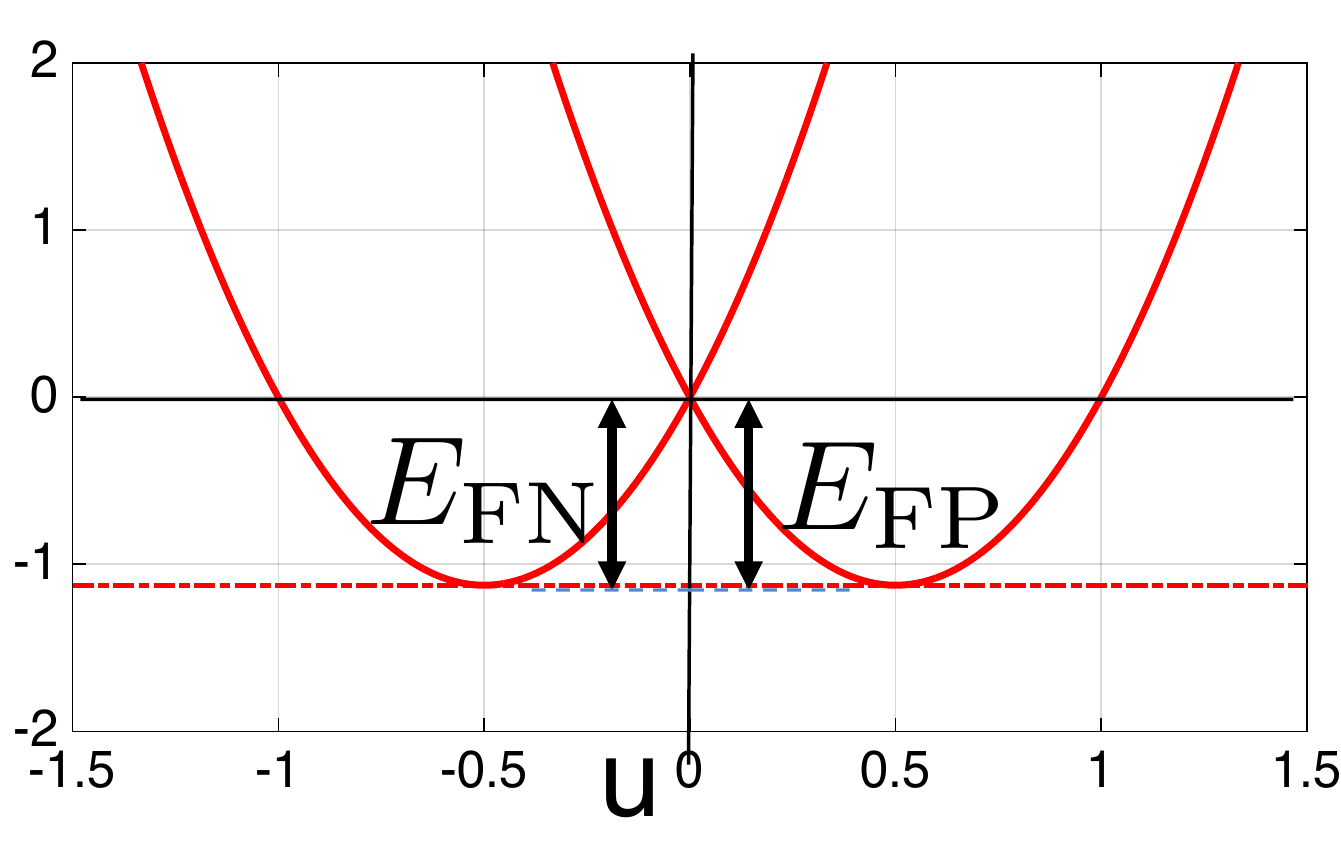}
\includegraphics[height=2.4cm]{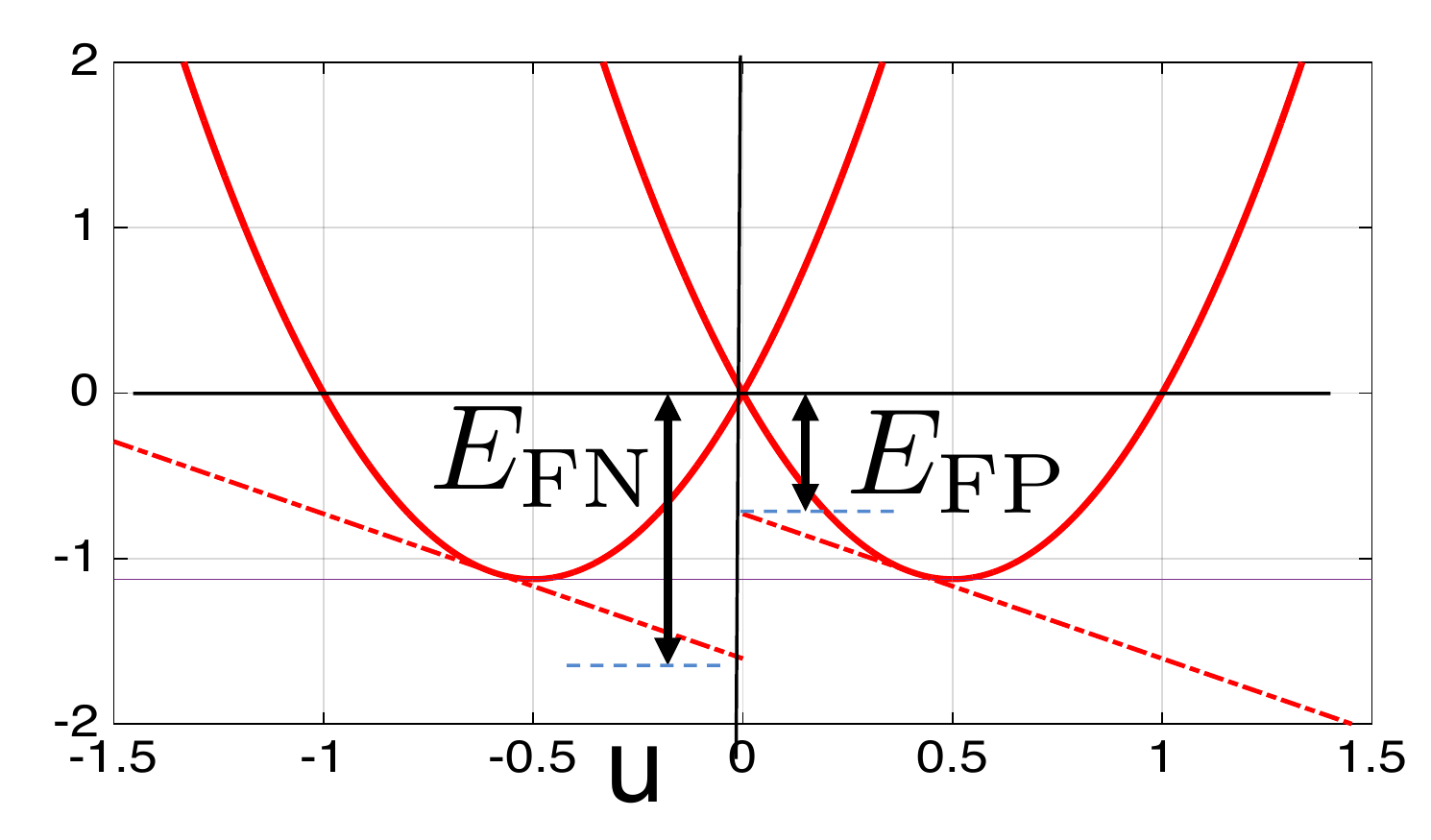}
\caption{Let $P_0(x){\sim}\mathcal{N}(1,1)$ and $P_1(x){\sim} \mathcal{N}(4,1)$. For a likelihood ratio detector $T(x){=}\log{\frac{P_1(x)}{P_0(x)}}\geq \tau$, we can compute the log-generating functions as follows:
$\Lambda_0(u)=\frac{9}{2}u(u-1)$ and 
$\Lambda_1(u)=\frac{9}{2}u(u+1)$ (derived in Appendix A.3). Note that, $\Lambda_0(u)$ is strictly convex with zeros at $u{=}0$ and $u{=}1$, and $\Lambda_1(u){=}\Lambda_0(u+1)$. We obtain $E_{\mathrm{FP},T}(\tau)$ and $E_{\mathrm{FN},T}(\tau)$ as the negative of the y-intercepts for tangents to $\Lambda_0(u)$ and $\Lambda_1(u)$ respectively with slope $\tau$. As we vary the slope of the tangent $(\tau)$, there is a trade-off between $E_{\mathrm{FP},T}(\tau)$ and $E_{\mathrm{FN},T}(\tau)$ until they both become equal at $\tau=0$ (third figure from left). The value of the exponent at $\tau{=}0$ (negative of the y-intercepts for tangents with $0$-slope) is defined as the Chernoff Information, given by: $\mathrm{C}(P_0,P_1){:=}E_{\mathrm{FP},T}(0){=}E_{\mathrm{FN},T}(0),$ which is equal to $9/8$ for this particular example.  \label{fig:exponents}} 
\end{figure*}
For each group $Z=z$, we will be denoting classifiers as $T_z(x)\geq\tau_z$, i.e., the prediction label is $1$ when $T_z(x)\geq \tau_z$ and $0$ otherwise. 
\begin{rem}[Decoupled Classifiers]
While such models may exhibit disparate treatment (explicit use of $Z$), the intent is to better mitigate disparate impact using the protected attribute explicitly in the decision making (along the spirit of fair affirmative action~\cite{dwork2012fairness,decoupled}). Furthermore, a classifier that does not use $Z$ becomes a special case of our classifier if $T_z$ and $\tau_z$ are the same for both groups. 
\end{rem}

Next, we state two basic assumptions: (\textbf{A1}) Absolute Continuity: $P_0(x)$, $P_1(x)$, $Q_0(x)$ and $Q_1(x)$ are greater than $0$ everywhere in the range of $x$. This ensures that likelihood ratio detectors such as $\log{\frac{P_1(x)}{P_0(x)}}\geq \tau_0$ and Kullback-Leibler (KL) divergences between any two of these distributions are well-defined.\footnote{Without this assumption, the definition of separability, i.e., Chernoff information (as we define later in \Cref{defn:separability}) can become infinite, and the problem is ill-posed.} (\textbf{A2}) Distinct Hypotheses: $\mathrm{D}(P_0{||}P_1)$, $\mathrm{D}(P_1{||}P_0)$, $\mathrm{D}(Q_0{||}Q_1)$ and $\mathrm{D}(Q_1{||}Q_0)$ are strictly greater than $0$, where $\mathrm{D}(\cdot{||}\cdot)$ is the KL divergence.

We let $P_{\mathrm{FP},T_z}(\tau_z)$ be the probability of false positive (wrongful acceptance of negative class labels; also called false positive rate (FPR)) over the group $Z=z$, i.e., 
$
P_{\mathrm{FP},T_z}(\tau_z) 
=\Pr{(T_z(X)\geq \tau_z|Y=0,Z=z)}. $
Similarly, $P_{\mathrm{FN},T_z}(\tau_z)$ is the probability of false negative (wrongful rejection of positive class labels; also called false negative rate (FNR)), given by: 
$
P_{\mathrm{FN},T_z}(\tau_z) 
=\Pr{(T_z(X) < \tau_z|Y=1,Z=z)}. $
The overall probability of error of a group is 
given by: 
$
P_{e,T_z}(\tau_z)
=\pi_0 P_{\mathrm{FP},T_z}(\tau_z)+ \pi_1 P_{\mathrm{FN},T_z}(\tau_z),$
where $\pi_0$ and $\pi_1$ are the prior probabilities of $Y=0$ and $Y=1$ given $Z=z$. For the sake of simplicity, we consider the case where $\pi_0=\pi_1=\frac{1}{2}$ given $Z=z$, and also equal priors on all groups $Z=z$. We include a discussion on how to extend our results for the case of unequal priors in Appendix E. Equal priors also correspond to the balanced accuracy measure~\cite{brodersen2010balanced} which is often favored over ordinary accuracy.

A well-known definition of fairness is \emph{equalized odds}~\cite{hardt2016equality}, which states that an algorithm is fair if it has equal probabilities of false positive (wrongful acceptance of true negative class labels) and false negative (wrongful acceptance of true positive class labels) for the two groups, i.e., $Z=0$ and $1$. A relaxed variant of this measure, widely used in the literature, is \emph{equal opportunity}, which enforces only equal false negative rate (or equivalently, equal true positive rate) for the two groups. In this work, we focus primarly on equal opportunity, although the arguments can be extended to other measures of fairness as well, e.g., statistical parity~\cite{agarwal2018reductions}. 

We assume that in the construct space, there is no trade-off between accuracy and equal opportunity, i.e., the Bayes optimal~\cite{cover2012elements} classifiers for the groups $Z=0$ and $Z=1$ also satisfy equal opportunity (equal probabilities of false negative). In this work, our objective is to explain the accuracy-fairness trade-off in the observed space and attempt to find ideal distributions with respect to which there is no trade-off. We now provide a brief background on error exponents of a classifier to help follow the rest of the paper.

\textbf{Background on Error Exponents of a Classifier:}
The error exponents of the FPR and FNR are given by $-\log{P_{\mathrm{FP},T_z}(\tau_z)}$ and $-\log{P_{\mathrm{FN},T_z}(\tau_z)}$. Often, we may not be able to obtain a closed-form expression for the exact error probabilities or their exponents, but the exponents are approximated using a well-known lower bound called the \emph{Chernoff bound} (see \Cref{lem:chernoff_exponent}; proof in Appendix A.1), that is known to be pretty tight~(see Remark~\ref{rem:tightness} and also \cite{motwani1995randomized,berend2015finite}). 

\begin{defn}
The Chernoff exponents of $P_{\mathrm{FP},T_z}(\tau_z)$ and $P_{\mathrm{FN},T_z}(\tau_z)$ are defined as: 
$$
E_{\mathrm{FP},T_z}(\tau_z)  =  \sup_{u{>}0} (u\tau_z - \Lambda_0(u)), \text{ and } $$
$$ E_{\mathrm{FN},T_z}(\tau_z) = \sup_{u<0} (u\tau_z - \Lambda_1(u)). $$
Here, $\Lambda_0(u)$ and $\Lambda_1(u)$ are called log-generating functions, given by $\Lambda_0(u) =  \log{\mathbb{E}[e^{u T_z(X)}|Y=0,Z=z] }$ and $\Lambda_1(u) =  \log{\mathbb{E}[e^{u T_z(X)}|Y=1,Z=z] }. $ \label{defn:chernoff}
\end{defn}
\begin{lem}[Chernoff Bound] The exponents satisfy:
$P_{\mathrm{FP},T_z}(\tau_z) {\leq} e^{-E_{\mathrm{FP},T_z}(\tau_z)}$ and
$P_{\mathrm{FN},T_z}(\tau_z) {\leq} e^{-E_{\mathrm{FN},T_z}(\tau_z)}.$\label{lem:chernoff_exponent}
\end{lem}
\begin{rem}[Tightness of the Chernoff Bound]
\label{rem:tightness}
For Gaussian distributions, the tail probabilities are characterized by the Q-function which has both upper and lower bounds in terms of Chernoff exponents with constant factors that do not affect the exponent significantly~\cite{cote2012chernoff}. The Bhattacharya bound (a special case of Chernoff bound) both upper and lower bounds the Bayes error exponent~\cite{berisha2015empirically,bhattacharyya1946measure,kailath1967divergence}. 
\end{rem}

\textbf{Geometric Interpretation of Chernoff Exponents:}  
Chernoff exponents yield more insight than exact error exponents because of their geometric interpretation, as we discuss here (more details in Appendix A.2). 

For ease of understanding, we refer to Fig.~\ref{fig:exponents} where we illustrate the idea of Chernoff exponents with a numerical example. In general, the log-generating functions are convex and become $0$ at $u=0$ (see Appendix A.2). Furthermore, if a detector is well-behaved\footnote{For a detector $T_z(x){\geq}\tau_z,$ we would expect $T_z(X)$ to be high when $Y{=}1$, and low when $Y{=}0$ justifying the criteria $\mathbb{E}[T_z(X)|Y{=}1,Z{=}z]{>}0$ and $ \mathbb{E}[T_z(X)|Y{=}0,Z{=}z]{<}0$ for being well-behaved. A likelihood ratio detector $T_0(x){=}\log{\frac{P_1(x)}{P_0(x)}}{\geq} \tau_0$ is well-behaved under assumption A2 in Section~\ref{sec:preliminaries} because we have $\mathbb{E}[T_z(X)|Y{=}1,Z{=}z]{=}D(P_1||P_0)$ and $\mathbb{E}[T_z(X)|Y{=}0,Z{=}z]{=}-D(P_0||P_1)$. }, i.e., $\mathbb{E}[T_z(X)|Y{=}1,Z{=}z]{>}0$ and $ \mathbb{E}[T_z(X)|Y{=}0,Z{=}z]{<}0$, then $\Lambda_0(u)$ and $\Lambda_1(u)$ are strictly convex and attain their minima on either sides of the origin. The Chernoff exponents $E_{\mathrm{FP},T_z}(\tau_z)$ and $E_{\mathrm{FN},T_z}(\tau_z)$ can be obtained as the negative of the y-intercepts for tangents to $\Lambda_0(u)$ and $\Lambda_1(u)$ with slope $\tau_z$ (for $\tau_z \in (\mathbb{E}[T_z(X)|Y{=}0,Z{=}z],\mathbb{E}[T_z(X)|Y{=}1,Z{=}z])$). 

\begin{defn} The Chernoff exponent of the overall probability of error ${P_{e,T_z}(\tau_z)}$ is defined as: 
\begin{equation}E_{e,T_z}(\tau_z)= \min \{  E_{\mathrm{FP},T_z}(\tau_z), E_{\mathrm{FN},T_z}(\tau_z) \}.\nonumber\end{equation} 
\end{defn}

Recall that, under equal priors, we have $P_{e,T_z}(\tau_z)=\frac{1}{2}P_{\mathrm{FP},T_z}(\tau_z)+\frac{1}{2}P_{\mathrm{FN},T_z}(\tau_z)$.  The exponent of $P_{e,T_z}(\tau_z)$ is dominated by the minimum of the error exponents of $P_{\mathrm{FP},T_z}(\tau_z)$ and $P_{\mathrm{FN},T_z}(\tau_z)$, which in turn is bounded by the minimum of the Chernoff exponents of FPR and FNR (\Cref{defn:chernoff}). A higher $E_{e,T_z}(\tau_z)$ indicates higher accuracy, i.e., lower $P_{e,T_z}(\tau_z)$. To understand this, first consider likelihood ratio detectors of the form $T_0(x)=\log{\frac{P_1(x)}{P_0(x)}}$ for $Z=0$. As we vary $\tau_0$, there is a trade-off between $P_{\mathrm{FP},T_0}(\tau_0)$ and $P_{\mathrm{FN},T_0}(\tau_0)$, i.e., as one increases, the other decreases. A similar trade-off is also observed in their Chernoff exponents (see \Cref{fig:exponents}). $P_{e,T_0}(\tau_0)$ is minimized when $\tau_0=0$ (for equal priors) and $P_{\mathrm{FP},T_0}(0){=}P_{\mathrm{FN},T_0}(0)$.  For this optimal value of $\tau_0=0$, the Chernoff exponents of FPR and FNR also become equal, i.e., $E_{\mathrm{FP},T_0}(0){=}E_{\mathrm{FN},T_0}(0)$, and the maximum value of $E_{e,T_0}(\tau_0){=}\min \{  E_{\mathrm{FP},T_0}(\tau_0), E_{\mathrm{FN},T_0}(\tau_0) \}$ is attained. This exponent is called the Chernoff information~\cite{cover2012elements}. For completeness, we include a well-known result on Chernoff information from \cite{cover2012elements} with the proof in Appendix A.4. 

\begin{lem}
\label{lem:separability}
For two hypotheses $P_0(x)$ under $Y=0$ and $P_1(x)$ under $Y=1$, the Chernoff exponent of the probability of error of the Bayes optimal classifier is given by the Chernoff information:\footnote{When $P_0(x)$ and $P_1(x)$ are continuous distributions, the summation is replaced by an integral over $x$ (see Appendix A.3).}  \begin{equation}\mathrm{C}(P_0,P_1)=- \min_{ u \in (0,1)} \log{\biggl(\sum_{x} P_0(x)^{1-u}P_1(x)^{u}\biggr)}. 
\end{equation} 
\end{lem}
\textbf{Goals:}
Our metrics of interest for \emph{accuracy} are  $E_{e,T_0}(\tau_0)$ and $E_{e,T_1}(\tau_1)$ because a higher value of the Chernoff exponent of $P_{e,T_z}(\tau_z)$ implies a higher accuracy for the respective groups $Z{=}0$ and $Z{=}1$. Our metric of interest for \emph{fairness} is the difference of the Chernoff exponents of FNR, i.e., $|E_{\mathrm{FN}, T_0}(\tau_0)- E_{\mathrm{FN},T_1}(\tau_1)|$ (inspired from equal opportunity). A model is \emph{fair} when $|E_{\mathrm{FN}, T_0}(\tau_0)- E_{\mathrm{FN},T_1}(\tau_1)|=0$, and progressively becomes more and more unfair as this quantity $|E_{\mathrm{FN}, T_0}(\tau_0)- E_{\mathrm{FN},T_1}(\tau_1)|$ increases.

Our first goal is to quantify fundamental limits on the best accuracy-fairness trade-off in terms of our metrics of interest on an existing real-world dataset, i.e., given observed distributions $P_0(x)$, $P_1(x)$, $Q_0(x)$, and $Q_1(x)$. Next, our goal is to find ideal distributions where fairness and accuracy are in accord when accuracy is measured with respect to the ideal distributions.

\section{Main Results}

\subsection{Concept of Separability: Fundamental Limits on Accuracy-Fairness Trade-Off in the Real World}
\label{subsec:limit}

Given the setup in Section~\ref{sec:preliminaries}, we show that the trade-off between accuracy and equal opportunity in the observed space is due to noisier mappings for the unprivileged group making their positive and negative labels less separable. Let us first formally define our intuitive notion of separability.

\begin{defn} For a group of people with distributions $P_0(x)$ and $P_1(x)$  under hypotheses $Y{=}0$ and $Y{=}1$, we define the separability as their Chernoff information $\mathrm{C}(P_0,P_1)$. \label{defn:separability}
\end{defn}

\Cref{defn:separability} is motivated from \Cref{lem:separability} because Chernoff information essentially provides an information-theoretic approximation to the best classification accuracy (in an exponent sense) for a group of people in a given dataset. Next, we define unbiased mappings from a separability standpoint.
\begin{defn}Consider the setup in \Cref{sec:preliminaries}. The mapping $X=f_{Y,Z}(X_a)$ from the construct space to the observed space is said to be unbiased  if $\mathrm{C}(P_0,P_1)=\mathrm{C}(Q_0,Q_1).$ \label{defn:unbiased_mappings}
\end{defn}
Our next result demonstrates that the trade-off between fairness and accuracy arises due to a bias in the mappings from a separability standpoint, i.e., $\mathrm{C}(P_0,P_1)\neq \mathrm{C}(Q_0,Q_1).$ Because we assumed that $Z=0$ is the unprivileged group, we let $\mathrm{C}(P_0,P_1)$ be either equal to, or less than $\mathrm{C}(Q_0,Q_1)$.

\begin{thm}[Explaining the Trade-Off]
\label{thm:separability} 
For the setup in \Cref{sec:preliminaries}, one of the following is true:
\begin{enumerate}[leftmargin=*,topsep=0pt, itemsep=0pt]
\item Unbiased Mappings, i.e., $\mathrm{C}(P_0,P_1){=} \mathrm{C}(Q_0,Q_1)$: The Bayes optimal detectors $T_0(x)\geq \tau_0$ and $T_1(x)\geq \tau_1$ for the two groups with Chernoff exponents of the probability of error $\mathrm{C}(Q_0,Q_1)(=\mathrm{C}(P_0,P_1))$ also attain fairness, i.e., $|E_{\mathrm{FN}, T_0}(\tau_0)- E_{\mathrm{FN},T_1}(\tau_1)|=0$. 
\item Biased Mappings, i.e., $\mathrm{C}(P_0,P_1)< \mathrm{C}(Q_0,Q_1)$: The Bayes optimal detectors $T_0(x)\geq \tau_0$ and $T_1(x)\geq \tau_1$ for the two groups are not fair, i.e., $|E_{\mathrm{FN}, T_0}(\tau_0)- E_{\mathrm{FN},T_1}(\tau_1)|\neq 0$. Furthermore, no likelihood ratio detector can improve the Chernoff exponent of the probability of error for the unprivileged group beyond $\mathrm{C}(P_0,P_1)$.
\end{enumerate}
\end{thm}

The first scenario is where the mappings are unbiased from a separability standpoint, and there is no trade-off between accuracy and fairness. The second scenario, which occurs more commonly in practice, is where discrimination is caused due to an inherent limitation of the dataset: the mappings from the construct space are biased and do not have enough separability information about one group compared to the other. For the rest of the paper, we will focus on the case of $\mathrm{C}(P_0,P_1)<\mathrm{C}(Q_0,Q_1)$. Under this scenario, the Chernoff exponents of FNR of the Bayes optimal detectors for the two groups are $\mathrm{C}(P_0,P_1)$ and $\mathrm{C}(Q_0,Q_1)$ which are unequal, and hence \emph{unfair}. An attempt to ensure fairness by using any alternate likelihood ratio detector for any of the groups will therefore only reduce accuracy (Chernoff exponent of the probability of error) for that group below the Bayes optimal (best) classifier for that group, explaining the accuracy-fairness trade-off. We formalize this intuition in Lemma~\ref{lem:trade-off} (used in proof of \Cref{thm:separability}; see Appendix B).
\begin{lem} 
\label{lem:trade-off}Let $\mathrm{C}(P_0,P_1){<}\mathrm{C}(Q_0,Q_1)$. 
Suppose that there are two likelihood ratio detectors $T_0(x){\geq} \tau_0$ and $T_1(x){\geq} \tau_1$, one for each group, such that $E_{\mathrm{FN},T_0}(\tau_0){=} E_{\mathrm{FN},T_1}(\tau_1).$ Then, at least one of the following statements is true:\\ (i) $E_{e,T_0}(\tau_0) < \mathrm{C}(P_0,P_1)$, or (ii) $E_{e,T_1}(\tau_1) < \mathrm{C}(Q_0,Q_1)$. 
\end{lem} 
The next two results show how current and reasonable approaches to fair classification can give rise to each of the two cases in Lemma~\ref{lem:trade-off}. 
\begin{figure*}
\centering
\includegraphics[height=3.7cm]{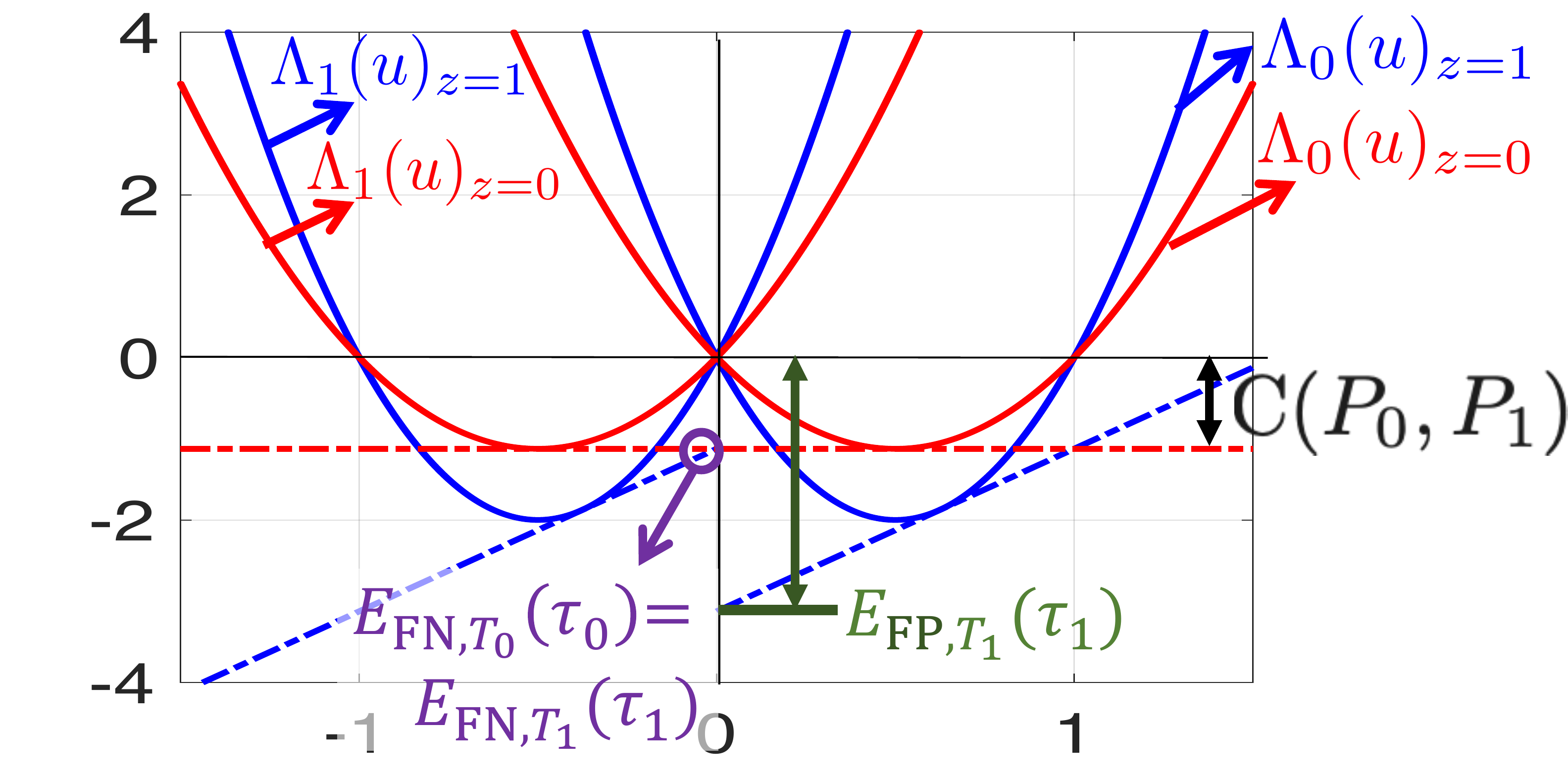}
\includegraphics[height=3.7cm]{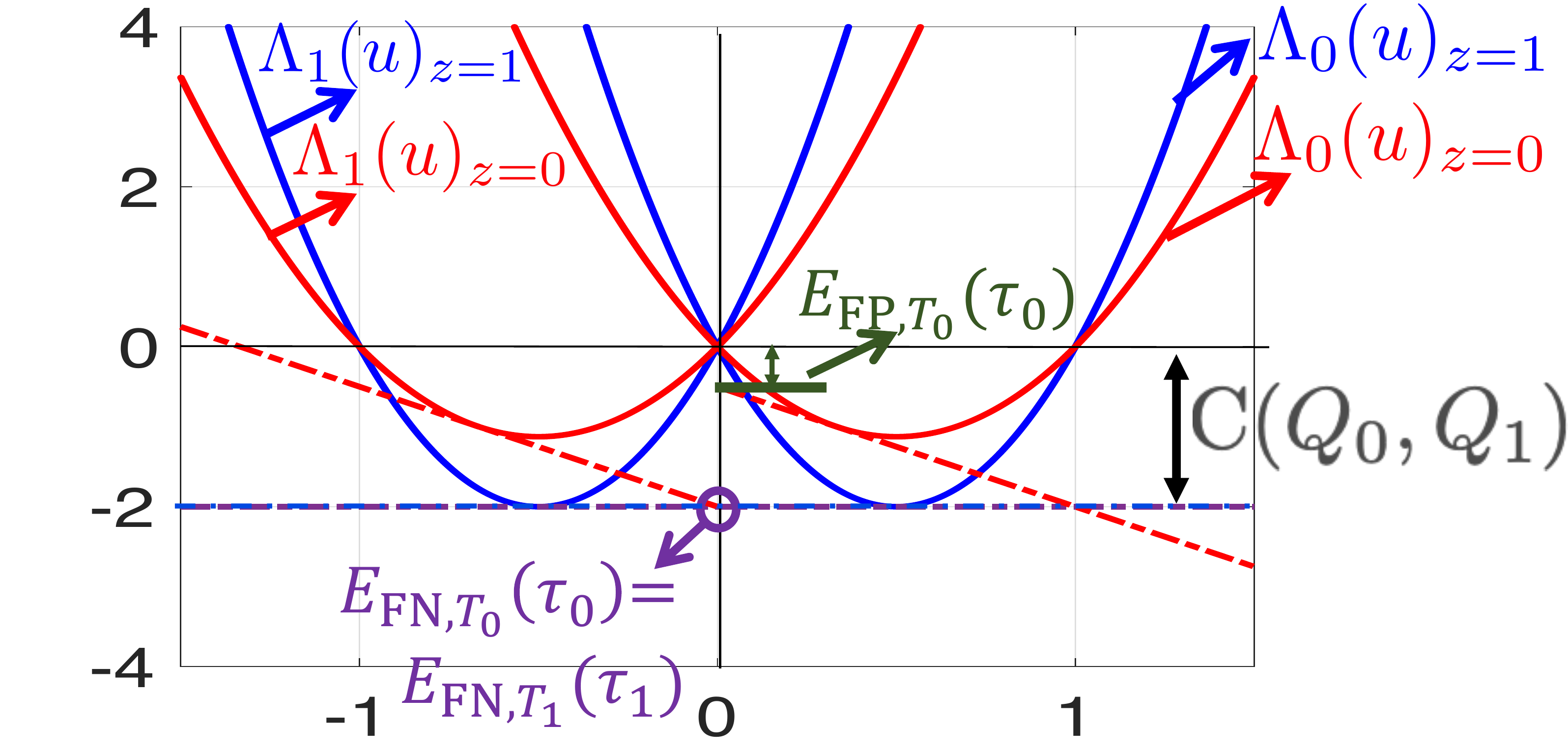}
\caption{Let the distributions for the unprivileged group ($Z=0$) be $P_0(x){\sim}\mathcal{N}(1,1)$ and $P_1(x){\sim} \mathcal{N}(4,1)$. Also, let the distributions of the privileged group be $Q_0(x){\sim}\mathcal{N}(0,1)$ and $Q_1(x){\sim} \mathcal{N}(4,1)$. In both the figures, the red and blue curves denote the log-generating functions for the likelihood ratio detectors for the groups $Z=0$ and $Z=1$ respectively (see Appendix A.3 for derivation). We have $\Lambda_0(u)_{z=1}=8u(u-1)$ and $\Lambda_1(u)_{z=1}=8u(u+1)$. Also, $\Lambda_0(u)_{z=0}=\frac{9}{2}u(u-1),$ and $\Lambda_1(u)_{z=0}=\frac{9}{2}u(u+1)$. Note that,  $\mathrm{C}(P_0,P_1){<}\mathrm{C}(Q_0,Q_1)$. (\textbf{Left}) This plot corresponds to the scenario of \Cref{lem:both_group}. The detector for the group $Z=0$ is the Bayes optimal detector with $\tau^*_0=0$ and  $E_{\mathrm{FN},T_0}(\tau^*_0)=E_{\mathrm{FP},T_0}(\tau^*_0)=C(P_0,P_1).$ The detector for the group $Z=1$ is a sub-optimal detector because in order to satisfy equal opportunity, we have to choose $\tau^*_1$ such that $E_{\mathrm{FN},T_1}(\tau^*_1) = E_{\mathrm{FN},T_0}(\tau^*_0)= C(P_0,P_1)$ and this is strictly less than $C(Q_0,Q_1)$. (\textbf{Right}) This plot corresponds to the scenario of \Cref{lem:one_group}. The detector for the group $Z=1$ is the Bayes optimal detector with $\tau^*_1=0$ and  $E_{\mathrm{FN},T_1}(\tau^*_1)=E_{\mathrm{FP},T_1}(\tau^*_1)=C(Q_0,Q_1).$ In order to satisfy equal opportunity, we have to choose $\tau^*_0$ such that $E_{\mathrm{FN},T_0}(\tau^*_0) = E_{\mathrm{FN},T_1}(\tau^*_1)= C(Q_0,Q_1)$ which is strictly greater that $C(P_0,P_1)$. However, this threshold $\tau^*_0$ makes $E_{\mathrm{FP},T_0}(\tau^*_0)$ lower that $C(P_0,P_1)$, leading to a sub-optimal detector for the group $Z=0$.  \label{fig:trade-off} }
\end{figure*}
Consider the following optimization problem, where the goal is to find classifiers of the form $T_0(x)\geq\tau_0$ and $T_1(x)\geq\tau_1$ for the two groups that maximize the Chernoff exponent of the probability of error under the constraint that they are \emph{fair} on the given dataset. 
\begin{align}
\begin{split}
&\max_{T_0,\tau_0,T_1,\tau_1}  \min\left\{ E_{\mathrm{FP},T_0}(\tau_0),E_{\mathrm{FN},T_0}(\tau_0),\right.\\ &\hspace{2.3cm} \left. E_{\mathrm{FP},T_1}(\tau_1),E_{\mathrm{FN},T_1}(\tau_1)\right\}
\end{split}\nonumber \\
&\text{such that } E_{\mathrm{FN},T_0}(\tau_0)=E_{\mathrm{FN},T_1}(\tau_1).
\label{opt:both_group}
\end{align} 
This optimization is in the spirit of existing works \cite{zafar2017fairness,agarwal2018reductions,donini2018empirical,celis2019} that maximize accuracy under fairness constraints. From the NP Lemma, we know that given any classifier, there exists a likelihood ratio detector which is at least as good in terms of accuracy. If we restrict $T_0(x)$ and $T_1(x)$ to be likelihood ratio detectors of the form $\log{\frac{P_1(x)}{P_0(x)}}$ and $\log{\frac{Q_1(x)}{Q_0(x)}}$, then \eqref{opt:both_group} has a unique solution $(\tau_0^*,\tau_1^*)$. 

\begin{lem} Let $\mathrm{C}(P_0,P_1){<}\mathrm{C}(Q_0,Q_1)$ and $T_0(x)$ and $T_1(x)$ be restricted to be likelihood ratio detectors. Then the detectors $T_0(x)\geq\tau^*_0$ and $T_1(x)\geq\tau^*_1$ that solve the optimization \eqref{opt:both_group}  
are the Bayes optimal detector for the unprivileged group ($\tau_0^*=0$) and a sub-optimal detector for the privileged group ($\tau_1^*>0$) with $E_{e,T_1}(\tau^*_1) < \mathrm{C}(Q_0,Q_1)$. \label{lem:both_group}
\end{lem}
As a proof sketch, we refer to \Cref{fig:trade-off} (Left). Let $\tau_0^*=0$, which ensures $E_{\mathrm{FN},T_0}(0)=E_{\mathrm{FP},T_0}(0)=\mathrm{C}(P_0,P_1)$. Now, the only value of slope $\tau_1^*$ that will satisfy $E_{\mathrm{FN},T_1}(\tau_1^*){=}E_{\mathrm{FN},T_0}(0)$ is a $\tau_1^*{>}0$ such that $E_{\mathrm{FN},T_1}(\tau_1^*){=}\mathrm{C}(P_0,P_1){<} \mathrm{C}(Q_0,Q_1),$ and hence $E_{\mathrm{FP},T_1}(\tau_1^*){>} \mathrm{C}(Q_0,Q_1).$ This leads to, $$\min\{ E_{\mathrm{FP},T_0}(0),E_{\mathrm{FN},T_0}(0), E_{\mathrm{FP},T_1}(\tau^*_1),E_{\mathrm{FN},T_1}(\tau^*_1)\}= \mathrm{C}(P_0,P_1).$$ 

For $\tau_0^*{\neq}0$, either $E_{\mathrm{FP},T_0}(\tau_0^*){<} \mathrm{C}(P_0,P_1){<} E_{\mathrm{FN},T_0}(\tau_0^*),$ or $E_{\mathrm{FN},T_0}(\tau_0^*){<} \mathrm{C}(P_0,P_1){<} E_{\mathrm{FP},T_0}(\tau_0^*),$ implying that, $$\min\{ E_{\mathrm{FP},T_0}(\tau_0^*),E_{\mathrm{FN},T_0}(\tau_0^*), E_{\mathrm{FP},T_1}(\tau^*_1),E_{\mathrm{FN},T_1}(\tau^*_1)\}< \mathrm{C}(P_0,P_1).$$

This situation of reducing the accuracy of the privileged group is often interpreted as causing \emph{active harm} to the privileged group. To avoid causing active harm while satisfying a fairness criterion, we may also consider a variant where we do not alter the optimal detector (or accuracy) of the privileged group (i.e., $E_{\mathrm{FN},T_1}(\tau_1)=E_{\mathrm{FP},T_1}(\tau_1)=\mathrm{C}(Q_0,Q_1)$ for the privileged group), but only vary the detector for the unprivileged group to achieve fairness. We propose the following optimization: 
\begin{align}
&\max_{T_0,\tau_0}  \min\{ E_{\mathrm{FP},T_0}(\tau_0),E_{\mathrm{FN},T_0}(\tau_0) \} \nonumber \\
&\text{such that }  E_{\mathrm{FN},T_0}(\tau_0)=\mathrm{C}(Q_0,Q_1).\label{opt:one_group}
\end{align}
Again, if we restrict $T_0(x)$ to be a likelihood ratio detector, then there exists a unique solution $\tau_0^*$ to optimization \eqref{opt:one_group}.

\begin{lem} Let ${T_0(x)}=\log{\frac{P_1(x)}{P_0(x)}}$ and we have ${\mathrm{C}(P_0,P_1)}<{\mathrm{C}(Q_0,Q_1)}$. The detector $T_0(x)\geq\tau^*_0$ that solves optimization \eqref{opt:one_group} 
is a sub-optimal detector for the unprivileged group with $E_{e,T_0}(\tau^*_0)<\mathrm{C}(P_0,P_1)$. \label{lem:one_group} 
\end{lem}

As a proof sketch, we refer to \Cref{fig:trade-off} (Right). If we choose $\tau_0^*\neq 0$, we get a sub-optimal detector for the unprivileged group with $E_{e,T_0}(\tau^*_0)<\mathrm{C}(P_0,P_1)$. The full proofs for \Cref{lem:both_group,lem:one_group} are provided in Appendix B.3.

\begin{rem}[Equal priors on $Z$] Along the lines of balanced accuracy measures, the optimization assumes equal priors on $Z=0$ and $Z=1$ as well. We refer to Appendix E.2 for modification of the optimization to account for unequal priors on $Z=0$ and $Z=1$.
\end{rem}

\begin{rem}[Generalization to other fairness measures] While we focus on equal opportunity here, the idea extends to other fairness measures as well. For example, if the best likelihood detectors for each group, i.e., $T_0(x)\geq 0$ and $T_1(x)\geq 0$ do not satisfy statistical parity~\cite{agarwal2018reductions}, while there are other pairs of detectors for the two groups that do satisfy the criterion, then for at least one of the two groups, a sub-optimal detector is being used.
\end{rem}

\subsection{The Mismatched Hypothesis Testing Perspective: Ideal Distributions with no Accuracy-Fairness Trade-Off}
\label{subsec:accord}

Here, we will show that there exist ideal distributions such that fairness and accuracy are in accord. Since the trade-off arises due to insufficient separability of the unprivileged group in the observed space, we are specifically interested in finding ideal distributions for the unprivileged group that match the separability of the privileged, and the same detector that achieved fairness with sub-optimal accuracy in Lemma~\ref{lem:one_group} now achieves optimal accuracy with respect to the ideal distributions. We show the existence of such ideal distributions and also provide an explicit construction.

\begin{thm}[Existence of Ideal Distributions]
\label{thm:feasibility} For the setup in \Cref{sec:preliminaries}, let $\mathrm{C}(P_0,P_1)<\mathrm{C}(Q_0,Q_1)$. Let us choose the Bayes optimal detector $T_1(x)=\log{\frac{Q_1(x)}{Q_0(x)}}{\geq }0$ for the group $Z=1$. Then, for group $Z=0$, there exist $\widetilde{P}_0(x)$ and $\widetilde{P}_1(x)$ of the form $\widetilde{P}_0(x)=\frac{P_0(x)^{(1-w)}P_1(x)^w }{\sum_x P_0(x)^{(1-w)}P_1(x)^w }$  and $\widetilde{P}_1(x)=\frac{P_0(x)^{(1-v)}P_1(x)^v }{\sum_x P_0(x)^{(1-v)}P_1(x)^v }$ for $w,v\in \mathcal{R}$ such that:
\begin{itemize}[leftmargin=*,topsep=0pt,itemsep=0pt]
\item (Fairness on given data) The Bayes optimal detector for the ideal distributions, i.e., $\widetilde{T_0}(x){=}\log{\frac{\widetilde{P}_1(x)}{\widetilde{P}_0(x)}}{\geq }0$ is equivalent to the detector $T_0(x)=\log{\frac{P_1(x)}{P_0(x)}}{\geq }\tau_0^*$ of Lemma~\ref{lem:one_group} that satisfies equal opportunity on the given dataset, i.e., $E_{\mathrm{FN},T_0}(\tau_0)=E_{\mathrm{FN},T_1}(0)=\mathrm{C}(Q_0,Q_1)$.
\item (Accuracy and Fairness on ideal data) The Chernoff exponent of the probability of error of the Bayes optimal detector on the ideal distributions, i.e., $ \mathrm{C}(\widetilde{P}_0,\widetilde{P}_1)=\mathrm{C}(Q_0,Q_1)$, and is hence greater than $\mathrm{C}(P_0,P_1)$.
\end{itemize} 
\end{thm}

The proof is provided in Appendix C. The first criterion demonstrates that one can always find ideal distributions such that the \emph{fair} detector with respect to the given distributions (see Lemma~\ref{lem:one_group}) is in fact the Bayes optimal detector with respect to the ideal distributions. Note that there exist multiple pairs of $(v,w)$ such that 
$\widetilde{P}_0(x)=\frac{P_0(x)^{(1-w)}P_1(x)^w }{\sum_x P_0(x)^{(1-w)}P_1(x)^w }$  and $\widetilde{P}_1(x)=\frac{P_0(x)^{(1-v)}P_1(x)^v }{\sum_x P_0(x)^{(1-v)}P_1(x)^v }$ satisfy the first criterion of the theorem.

The second criterion goes a step further and demonstrates that among such pairs of ideal distributions, one can always find at least one pair such that they are just as separable as the privileged group (i.e.,$C(\widetilde{P}_0,\widetilde{P}_1)=C(Q_0,Q_1)$). The Bayes optimal detector for the unprivileged group with respect to the ideal distributions, i.e., $\widetilde{T_0}(x)=\log{\frac{\widetilde{P}_1(x)}{\widetilde{P}_0(x)}}{\geq }0$ is thus not only \emph{fair} on the given dataset but also satisfies equal opportunity on the ideal data because its Chernoff exponent of FNR is also equal to that of the privileged group, i.e., $C(Q_0,Q_1)$. Note that, in order to satisfy the second criterion, we restrict ourselves to choosing $v=1$ which leads to an appropriate value of $w$.

\begin{rem}[Uniqueness] Theorem~\ref{thm:feasibility} provides a proof of existence of ideal distributions along with an explicit construction. In general, there may exist other pairs of distributions, which are not of the particular form mentioned in \Cref{thm:feasibility}, but might satisfy the two conditions of the theorem. Therefore, given only $P_0(x)$ and $P_1(x)$, the ideal distributions are not necessarily unique unless further assumptions are made about their desirable properties. 
\end{rem}

In order to go about finding such ideal distributions in practice, we therefore propose an additional desirable property of such an ideal dataset. We require the ideal dataset to be a useful representative of the given dataset. This motivates a constraint that $\pi_0 \mathrm{D}(\widetilde{P}_0||P_0) + \pi_1 \mathrm{D}(\widetilde{P}_1||P_1)$ be as small as possible, i.e., the KL divergences of the ideal distributions from their respective given real-world distributions are small. Building on this perspective, we formulate the following optimization for specifying two ideal distributions $\widetilde{P}_0$ and $\widetilde{P}_1$ for the unprivileged group: 
\begin{align} 
\min_{\widetilde{P}_0,\widetilde{P}_1} & \pi_0 \mathrm{D}(\widetilde{P}_0||P_0) + \pi_1 \mathrm{D}(\widetilde{P}_1||P_1) \nonumber \\
 \text{such that, } &E_{\mathrm{FN},\widetilde{T_0}}(0)=\mathrm{C}(Q_0,Q_1),\label{opt:KL}
\end{align}
where $\widetilde{T_0}(x)=\log{\frac{\widetilde{P}_1(x)}{\widetilde{P}_0(x)}}\geq 0$ is the Bayes optimal detector with respect to the ideal distributions and $E_{\mathrm{FN},\widetilde{T_0}}(0)$ is the Chernoff exponent of the probability of false negative for this detector when evaluated on the given distributions $P_0(x)$ and $P_1(x)$. \Cref{thm:feasibility} already shows that the aforementioned optimization is feasible.

The results of this subsection can be extended to optimization \eqref{opt:both_group}, or to other measures of fairness altogether, e.g., statistical parity, or to other kinds of constraints such as minimal individual distortion.

\textbf{Relation to the construct space:} The ideal distributions for the unprivileged group, in conjunction with the given distributions of the privileged group, have two interpretations: (i) They could be viewed as plausible distributions in the observed space if the mappings were unbiased from a separability standpoint (recall \Cref{defn:unbiased_mappings}). (ii) Given our limited knowledge of the construct space, they could also be viewed as candidate distributions in the construct space itself if the mappings for the group $Z=1$ were identity mappings. This can be justified because we do not have much knowledge about the construct space (or even its dimensionality) except through the observed data. It is not unfathomable to assume they would have a separability of at least $\mathrm{C}(Q_0,Q_1)$, which is the separability exhibited by the privileged group in the observed space. Theorem~\ref{thm:feasibility} thus also demonstrates that the construct space is non-empty.

\begin{rem}[Explicit Use of an Ideal Dataset]
Several existing methods~\cite{CalmonWVRV2018,feldman2015certifying,kamiran2012data} propose pre-processing the given dataset to generate an alternate dataset that satisfies certain fairness and utility (representation) properties, in the same spirit as optimization~\eqref{opt:KL}, and train models on them. The trained detector may be sub-optimal with respect to the given dataset but is deemed to be fair. The results in this subsection help to explain why these approaches result in an 
accuracy-fairness trade-off on the given dataset, and also demonstrate that both accuracy and fairness can improve simultaneously when the accuracy is measured with respect to the alternate/ideal dataset. Optimization~\eqref{opt:KL} is also reminiscent of the formulation of \cite{jiang2019identifying}, who posit that a given biased label function is closest to an ideal unbiased label function in terms of KL divergence. In that work however, the KL divergence is applied to conditional label distributions $p_{Y|X}$ as opposed to conditional feature distributions $p_{X|Y}$. Furthermore, \cite{jiang2019identifying} do not analytically characterize trade-offs.
\end{rem}

\begin{rem}[Implicit Use of an Ideal Dataset]
Existing methods that fall in this category include training with fairness regularization in the loss function or post-processing the output to meet a fairness criterion. Instead of explicitly generating an ideal dataset, these methods aim to find a classifier that satisfies a fairness criterion on the given dataset, with minimal compromise of accuracy on the given dataset (recall optimizations \eqref{opt:both_group} and \eqref{opt:one_group}). Here, we show that there exist ideal distributions corresponding to these fair detectors such that a sub-optimal detector on the given dataset can be optimal with respect to the ideal dataset. 
\end{rem}

\subsection{Active Data Collection: Alleviating Real-World Trade-Offs with Improved Knowledge}
\label{subsec:explainability}

The inherent limitation of disparate separability between groups in the given dataset, discussed in \Cref{subsec:limit}, can in fact be overcome but with an associated cost: active data collection. In this section, we demonstrate when gathering more features can help in improving the Chernoff information of the unprivileged group without affecting that of the privileged group. Gathering more features helps us classify members of the unprivileged group more carefully with additional separability information that was not present in the initial dataset. In fact, this is the idea behind active fairness \cite{NoriegaCamperoBGP2019,BakkerNTSVP2019, chen2018my}. Our analysis below also serves as a technical explanation for the success of active fairness. {We note that while we discuss the scenario of additional data collection for the group $Z=0$ here, the result holds for any group or sub-group (also see Remark~\ref{rem:active_fairness_groups}).}

Let $X'$ denote the additional features so that $(X,X')$ is now used for classification of the group $Z{=}0$. Note that $X'$ could also easily be other forms of additional information including extra explanations to go along with the data or decision, similar to \cite{varshney2018interpretability}. Let $(X,X')$ have the following distributions: $(X,X')|_{Y=0,Z=0}\sim W_0(x,x')$ and $(X,X')|_{Y=1,Z=0}\sim W_1(x,x')$, where $Y$ is the true label. Note that, $P_0(x)=\sum_{x'}W_0(x,x')$ and $P_1(x)=\sum_{x'}W_1(x,x')$. Our goal is to derive the conditions under which the separability improves with addition of more features, i.e., $\mathrm{C}(W_0,W_1)>\mathrm{C}(P_0,P_1)$. 

\begin{thm}[Improving Separability]
The Chernoff information $\mathrm{C}(W_0,W_1)$ is strictly greater than $\mathrm{C}(P_0,P_1)$ if and only if $X'$ and $Y$ are not independent of each other given $X$ and $Z=0$, i.e., the conditional mutual information $I(X';Y|X,Z=0)>0$.
\label{thm:explainability} 
\end{thm}
The proof is provided in Appendix D. Note that, in general $\mathrm{C}(W_0,W_1)\geq \mathrm{C}(P_0,P_1) $ because intuitively separability can only improve or remain the same with additional findings (see Appendix D).
We attempt to identify the scenario where the inequality is strict.

Let ${x'}$ be a deterministic function of $x$, i.e., $f(x)$. Then
$W_0(x,x'){=}
P_0(x)$ if $x'{=}f(x)$, and $0$ otherwise. Similarly, $W_1(x,x'){=}
P_1(x)$ if $x'{=}f(x)$, and $0$ otherwise, leading to $\mathrm{C}(W_0,W_1){=}\mathrm{C}(P_0,P_1)$. This agrees with the intuition that if $X'$ is fully determined by $X$, then it does not improve the separability beyond what one could achieve using $X$ alone. Therefore, for $\mathrm{C}(W_0,W_1){>}\mathrm{C}(P_0,P_1)$, we require $X'$ to contribute some information that helps in separating hypotheses $Y=0$ and $Y=1$ better, that essentially leads to $X'$ not being independent of $Y$ given $X$ and $Z=0$. If new data improves the separability of the group $Z=0$, its accuracy-fairness trade-off is alleviated (see Fig.~\ref{fig:active} in Section~\ref{sec:numerical}).

\begin{rem}[Broader Interpretation of Active Data Collection] Active data collection can be interpreted as a more careful examination of a patient in healthcare applications, or a manual reconsideration before an automated rejection by an algorithm in hiring or lending applications, or examining any additional ``informative'' feature collected with the candidate's consent that improves decision making for societal welfare. One may argue that additional data collection from members of unprivileged groups as a way to improve their outcomes is an undue burden on them, and thus unfair. Although this may be true for \emph{unconsented} surveillance of entire populations, the allocation decisions that we investigate herein (e.g.\ hiring, lending, healthcare) tend to be ones in which the applicant willingly and consensually seeks an opportunity from an institution that has controls in place to deal with their data soundly. For example in hiring, it is the desire of applicants to progress from a simple resume check to an in-person interview; this opportunity allows them to provide more information so that their strengths can be understood by decision makers better (cf. the Rooney rule in hiring professional football coaches~\cite{emelianov2020fair}). Similarly in healthcare, if patients report pain symptoms, they would rather not be dismissed, but would like a physician to spend more time with them and conduct tests to obtain better diagnosis and care. In these and similar other examples, collecting additional data is a way to alleviate the trade-off. One may further argue that even in the consenting setting, additional data collection for members of unprivileged groups erects additional hoops for them to jump through, but in fact, additional features can always be collected for \textbf{all} groups of people.

We note that while we discuss additional feature collection for $Z=0$, active data collection can be performed for all groups/sub-groups of people as required by the application, and our results apply. E.g., if the collected data/features are more informative for any group/sub-group denoted by $Z=z'$ (i.e., $I(X';Y|X,Z=z')>0$), then they will improve the separability (Chernoff Information) of that group. New ideal distributions can also be found using the techniques of \Cref{subsec:accord} that are more plausible as both candidate observed-space distributions under unbiased mappings or candidate construct-space distributions. The new ideal distributions will have better separability if the new data improves the separability of all groups (elaborated further in Section~\ref{sec:discussion}). 
\label{rem:active_fairness_groups}
\end{rem}

\begin{figure}
\centerline{\includegraphics[height=5.3cm]{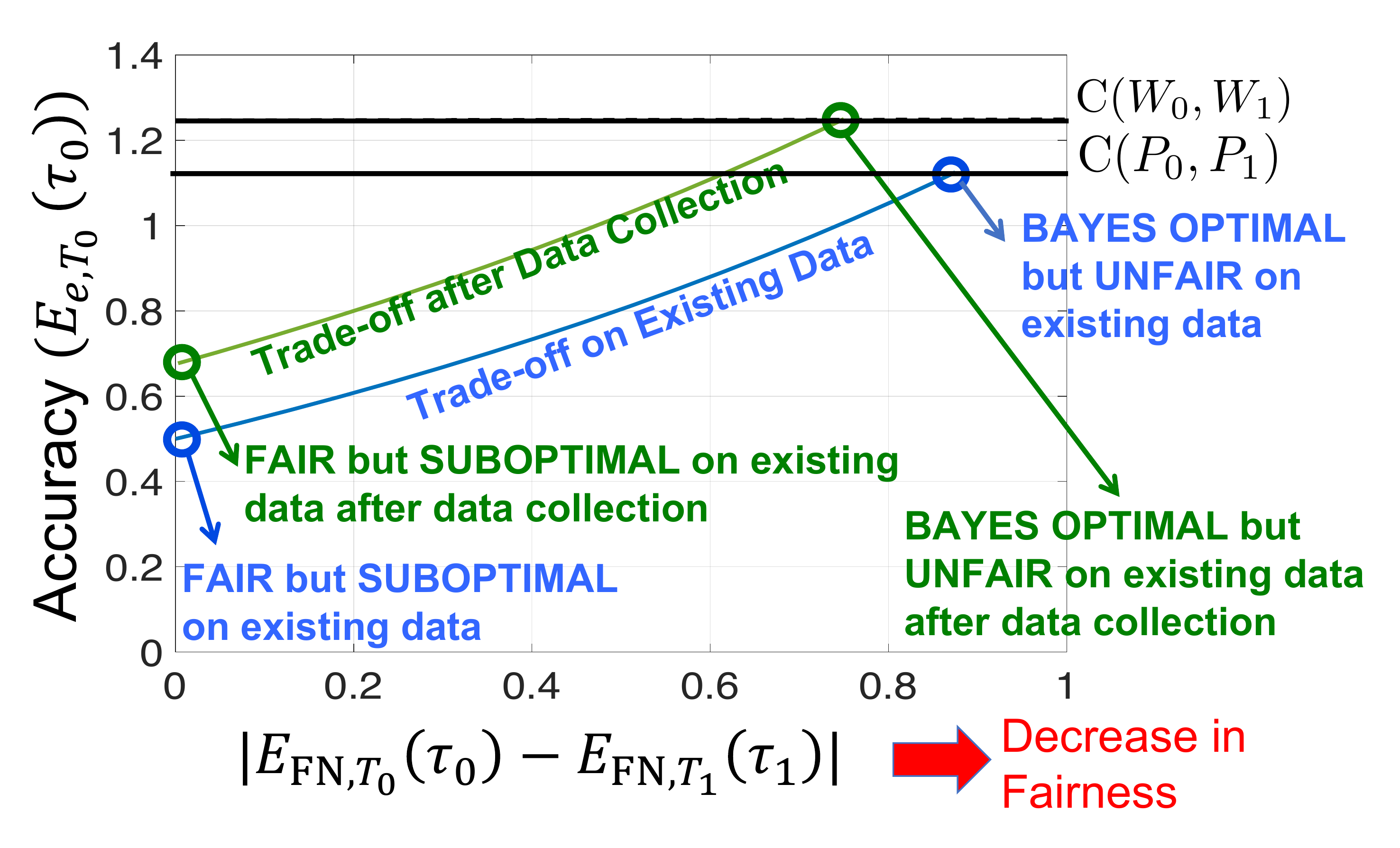}}
\caption{Computation of the trade-off between fairness and accuracy using a numerical example: For the unprivileged group, we let $P_0(x){\sim} \mathcal{N}(1,1)$ and $P_1(x){\sim} \mathcal{N}(4,1).$ We restrict the detector of the privileged group to its Bayes optimal detector with $C(Q_0,Q_1)=2$. The blue curve denotes the trade-off between accuracy and fairness in the existing dataset for the unprivileged group. Now suppose we are able to collect an additional feature $X'$ for the unprivileged group such that $(X,X')|_{Y=0,Z=0} \sim \mathcal{N}((1,1),\mathbf{I})$ and $(X,X')|_{Y=1,Z=0} \sim \mathcal{N}((4,2),\mathbf{I})$, where $\mathbf{I}$ is the $2\times 2$ identity matrix. The green curve shows how active data collection alleviates the trade-off between fairness and accuracy.  \label{fig:active} }
\end{figure}

\section{Numerical Example}
\label{sec:numerical}
We use a simple numerical example to show how our theoretical concepts and results can be computed in practice. 

\begin{eg}
\label{eg:numerical} Let the exam score for $Z=0$ be $P_0(x){\sim} \mathcal{N}(1,1)$ and $P_1(x){\sim} \mathcal{N}(4,1)$, and that for $Z=1$ be $Q_0(x){\sim}\mathcal{N}(0,1)$ and $Q_1(x){\sim}\mathcal{N}(4,1)$. 
\end{eg}

Let us restrict ourselves to likelihood ratio detectors of the form $T_0(x)=\log{\frac{P_0(x)}{P_1(x)}}\geq \tau_0$ and $T_1(x)=\log{\frac{Q_0(x)}{Q_1(x)}}\geq \tau_1$ for the two groups. The log generating functions for $Z=1$ can be computed analytically as:  $\Lambda_0(u)_{z=1}=8u(u-1)$ and $\Lambda_1(u)_{z=1}=8u(u+1)$ (derivation in Appendix A.3) and the Chernoff information can be computed as $\mathrm{C}(Q_0,Q_1)=2$. 

Now, for the unprivileged group $Z=0$, the log generating functions can be computed as $\Lambda_0(u)_{z=0}=\frac{9}{2}u(u-1)$ and $\Lambda_1(u)_{z=0}=\frac{9}{2}u(u+1)$ (again see Appendix A.3 for derivation). The Chernoff information is $\mathrm{C}(P_0,P_1)=9/8$.

\textbf{Accuracy-Fairness Trade-off in Real World:}   We restrict the detector for the privileged group to be the Bayes optimal detector $T_1(x){=}\log{\frac{Q_1(x)}{Q_0(x)}} \geq 0$ (equivalent to $x\geq 2$). For this detector, $E_{\mathrm{FP},T_1}(0){=}E_{\mathrm{FN},T_1}(0)= \mathrm{C}(Q_0,Q_1)=2$. 

Now, for $Z{=}0$, the Bayes optimal detector $T_0(x){=}\log{\frac{P_1(x)}{P_0(x)}}{\geq} 0$ (or, $x{\geq} 1.5$) will be unfair since $$E_{\mathrm{FN}, T_0}(0){=} \mathrm{C}(P_0,P_1) {<} E_{\mathrm{FN},T_1}(0).$$ Using the geometric interpretation of Chernoff information (recall Fig.~\ref{fig:trade-off}), we can compute the Chernoff exponents of FPR and FNR, i.e., $E_{\mathrm{FP},T_0}(\tau_0)$ and $E_{\mathrm{FN},T_0}(\tau_0)$ as the negative of the y-intercept of the tangents to $\Lambda_0(u)_{z=0}$ and $\Lambda_1(u)_{z=0}$ for detectors $T_0(x){=}\log{\frac{P_1(x)}{P_0(x)}}{\geq}\tau_0$. This enables us to numerically plot the trade-off between accuracy ($E_{e, T_0}(\tau_0){=}\min\{E_{\mathrm{FP}, T_0}(\tau_0),E_{\mathrm{FN}, T_0}(\tau_0) \}$) and fairness ($|E_{\mathrm{FN}, T_0}(\tau_0){-}E_{\mathrm{FN}, T_1}(\tau_0)|$) by varying $\tau_0$ as shown by the blue curve in Fig.~\ref{fig:active}. 

Note that, the detector that satisfies fairness (equal opportunity) on the given distributions can also be computed analytically as $\log{\frac{P_1(x)}{P_0(x)}}{\geq}\tau^*_0$ where $\tau_0^*{=}{-}3/2$ (equivalent to $x{\geq} 2$). This leads to equal exponent of FNR, i.e., $E_{\mathrm{FN}, T_0}(-3/2){=} 2{=}E_{\mathrm{FN},T_1}(0)$ but for this detector $E_{\mathrm{FP}, T_0}(\tau^*_0){=}1/2$ leading to reduced Chernoff exponent of overall error probability (represents accuracy), i.e., $E_{e, T_0}(\tau^*_0){=}\min\{E_{\mathrm{FP}, T_0}(\tau^*_0),E_{\mathrm{FN}, T_0}(\tau^*_0) \}=\min \{1/2,2\}=1/2$ which is less than $\mathrm{C}(P_0,P_1)=9/8$. 

\textbf{Ideal Distributions:} We refer to Fig.~\ref{fig:ideal}. It turns out that one pair of ideal distributions prescribed by Theorem~\ref{thm:feasibility} is $\widetilde{P}_0{=}Q_0$ and $\widetilde{P}_1{=}P_1{=}Q_1$. The Bayes optimal detector with respect to the ideal distributions for $Z=0$ is given by $\log{\frac{\widetilde{P}_1(x)}{\widetilde{P}_0(x)}} {\geq} 0$ (equivalent to $x{\geq} 2$). Note that, this is equivalent to the detector $\log{\frac{P_1(x)}{P_0(x)}}\geq \tau_0^*$ where $\tau_0^*{=}-3/2$ which satisfied equal opportunity on the given dataset. This detector is now Bayes optimal with respect to the ideal distributions $\widetilde{P}_0$ and $\widetilde{P}_1$, and has a Chernoff exponent of the overall probability of error equal to $C(\widetilde{P}_0,\widetilde{P}_1)=2$ when measured with respect to the ideal distributions. Thus, we demonstrate that both fairness (in the sense of equal opportunity on existing dataset as well as ideal dataset) and accuracy (with respect to the ideal distributions) are in accord. Note that, one may also find alternate pairs of ideal distributions using optimization \eqref{opt:KL} or any variant of the optimization, e.g., using statistical parity.

\begin{figure}
\centerline{\includegraphics[height=3cm]{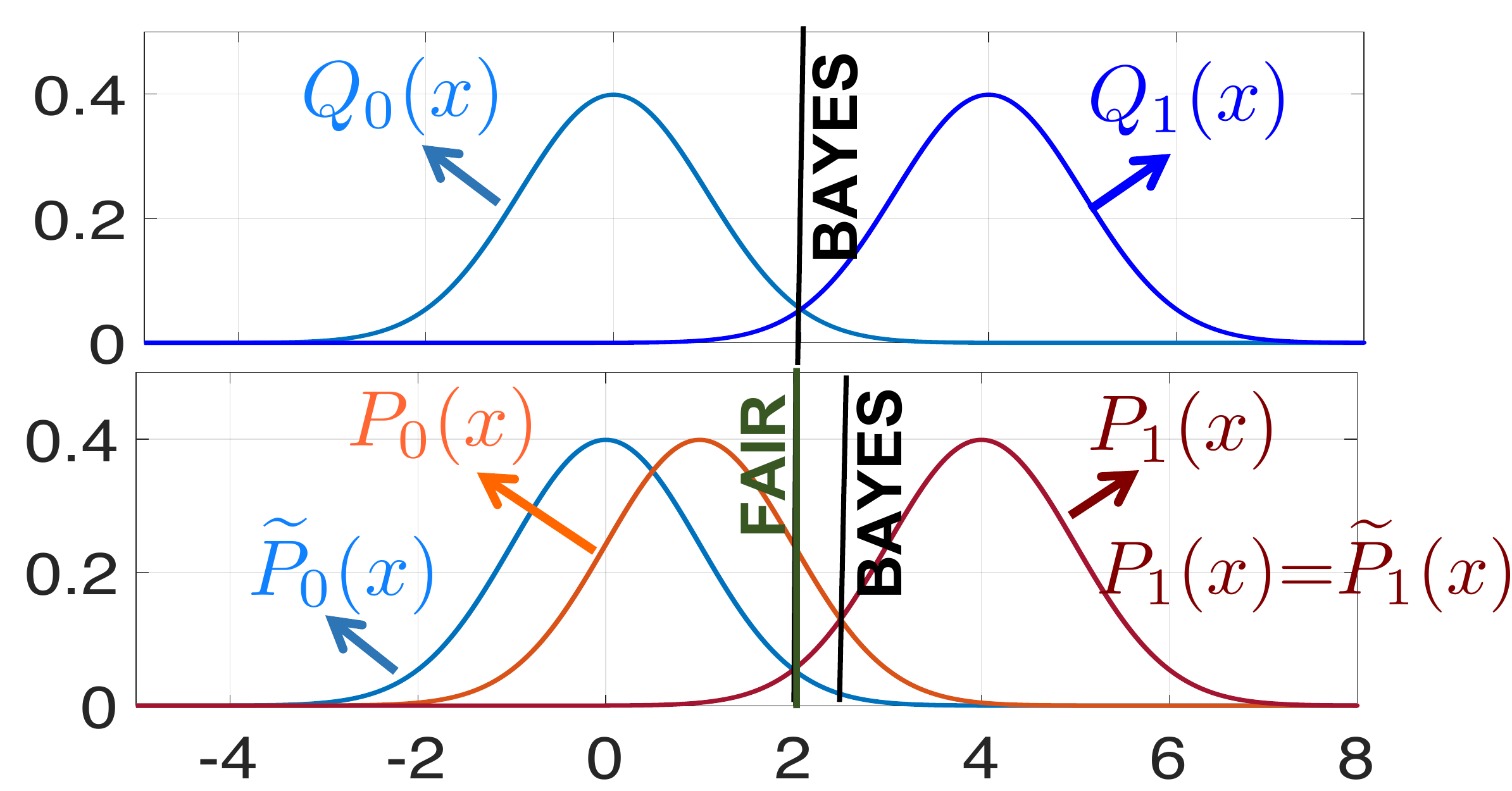}}
\caption{(Top) For the distributions in \Cref{eg:numerical}, we denote the Bayes optimal detector $\log{\frac{Q_1(x)}{Q_0(x)}}{\geq}0$ (equivalent to $x\geq2$) for the privileged group $Z=1$. (Bottom) For $Z=0$, the optimal detector $\log{\frac{P_1(x)}{P_0(x)}}{\geq}0$ does not satisfy equal opportunity on the given dataset but a sub-optimal detector does (notice the equal area corresponding to false negative rate for two groups). However, there exist ideal distributions given by $\widetilde{P}_0=Q_0$ and $\widetilde{P}_1=P_1=Q_1$ such that this detector is optimal w.r.t.\ the ideal distributions, and also achieves fairness w.r.t.\ both existing and ideal distributions. \label{fig:ideal}}
\end{figure}
\textbf{Active Data Collection:} Now suppose we are able to collect an additional feature $X'$ for $Z=0$ such that $(X,X')|_{Y=0,Z=0} \sim \mathcal{N}((1,1),\mathbf{I})$ and $(X,X')|_{Y=1,Z=0} \sim \mathcal{N}((4,2),\mathbf{I})$, where $\mathbf{I}$ is the $2\times 2$ identity matrix. The log generating functions can be derived as: $\Lambda_0(u)=5u(u-1)$ and $\Lambda_1(u)=5u(u+1).$ Note that, the Chernoff information (separability) $C(W_0,W_1)=5/4$ which is greater than $C(P_0,P_1)=9/8$. Thus, the collection of the new feature has improved the separability of the unprivileged group. 

Now, we examine the effect of active data collection on the accuracy-fairness trade-off in the real world. We again refer to Fig.~\ref{fig:active} (green curve). Consider the likelihood ratio detector for $Z=0$ based on the total set of features, i.e., $T_0(x,x')=\log\frac{W_0(x,x')}{W_1(x,x')} \geq \tau_0.$ To satisfy our fairness constraint, we need to choose a $\tau_0^*$ such that $E_{\mathrm{FN},T_0}(\tau_0^*){=}E_{\mathrm{FN},T_1}(0)= \mathrm{C}(Q_0,Q_1)=2$. Upon solving, we obtain that $\tau_0^*=5-\sqrt{40} \approx -1.32$. For this value of $\tau_0^*$, we obtain $E_{\mathrm{FP},T_0}(\tau^*_0){=}7-\sqrt{40})\approx 0.68.$ The Chernoff exponent of the probability of error for this \emph{fair} detector is given by $\min \{E_{\mathrm{FN},T_0}(\tau^*_0), E_{\mathrm{FP},T_0}(\tau^*_0) \} =\min\{2, 0.68\}=0.68$ which is greater than $0.5$ (the Chernoff exponent of the probability of error for the fair detector before collection of the additional feature $X'$).

\section{Discussion} 
\label{sec:discussion}
The trade-off between accuracy and fairness has been a topic of active debate in recent years. This work demystifies this problem by introducing the concept of separability: a quantification of the best accuracy attainable on a group of people. Separability can be viewed as the inherent ``informativeness'' of the data to be able to correctly make a particular decision. We assert that the trade-off between accuracy and fairness (equal opportunity) on observed datasets could be due to an inherent difference in informativeness regarding the two groups of people, possibly due to noisier representations for the unprivileged group due to historic differences in representation, opportunity, etc. Informativeness does not necessarily depend on how many features or data-points, e.g., even the distribution of a single relevant feature can be more informative than that of a bunch of less relevant features combined. This work examines the problem of accuracy-fairness trade-off from the lens of data informativeness.

We show that if there is a difference in separability (accuracy of the best classifiers) on two groups of people, then the best classifiers will not be fair on the given dataset. Any attempt to change the classifiers to attain fairness can affect the accuracy for one or both the groups. This intuition explains the observed trade-off on the given dataset. Our results also provide novel analytical insights that can quantify this accuracy-fairness trade-off approximately. Our Chernoff information based analysis can help estimate the respective separabilities on two groups of people, even before any classification algorithm is applied, albeit with estimation challenges~\cite{Chernoff1, Chernoff2, Chernoff3} that we are looking into as future work. 

We also show that there exist ideal distributions where fairness and accuracy can be in accord. Even the same classifier that compromises accuracy to attain fairness on a given dataset can attain both fairness and optimal accuracy on the ideal dataset. We believe that our demonstration that fairness and accuracy are in accord with respect to ideal datasets will motivate the use of accuracy with respect to an ideal dataset as a performance metric in algorithmic fairness research \cite{sharmadata,wick2019unlocking}. 

These ideal datasets also have intellectual connection with the framework of \cite{friedler2016possibility,yeom2018discriminative}. They can be viewed as unbiased distributions in the observed space if the mappings for both groups of people are equally noisy (or equally informative), or even ``candidate'' distributions in the ideal construct space itself. We note that for the latter interpretation though, we operate on an implicit assumption: the separability of both the groups in the construct space is equal to the highest separability among all the groups in the observed space. In essence, what this means is that the observed data for the group with highest Chernoff information is being assumed to be as informative as the ideal construct space data, so that further data collection cannot improve the separability beyond that (further implications of this assumption is revisited in the next paragraph).

Lastly, our results also show when and by how much  active data collection can alleviate the accuracy-fairness trade-off in the real world. If it turns out that active data collection improves the separability of all groups (including the group with highest Chernoff information), then it also helps us improve our working candidate distributions for the ideal construct space as well, i.e., we now know that the separability of all the groups in the ideal construct space is at least as much as the new highest Chernoff information among all the groups. We can therefore update our speculation about the ideal distributions for all groups accordingly.

\section*{Acknowledgements}

The authors would like to thank Pulkit Grover, Shubham Sharma, and the anonymous reviewers for their valuable suggestions. \nocite{gallager2012detection,boucheron2013concentration}

\bibliography{example_paper}

\begin{thebibliography}{10}

\bibitem{calmon2017optimized}
Flavio~P. Calmon, Dennis Wei, Bhanukiran Vinzamuri, Karthikeyan
  Natesan~Ramamurthy, and Kush~R. Varshney.
\newblock Optimized pre-processing for discrimination prevention.
\newblock In {\em Advances in Neural Information Processing Systems}, pages
  3992--4001, 2017.

\bibitem{dwork2012fairness}
Cynthia Dwork, Moritz Hardt, Toniann Pitassi, Omer Reingold, and Richard Zemel.
\newblock Fairness through awareness.
\newblock In {\em Proceedings of the Innovations in Theoretical Computer
  Science Conference}, pages 214--226, 2012.

\bibitem{agarwal2018reductions}
Alekh Agarwal, Alina Beygelzimer, Miroslav Dud{\'\i}k, John Langford, and Hanna
  Wallach.
\newblock A reductions approach to fair classification.
\newblock In {\em Proceedings of the International Conference on Machine
  Learning}, pages 60--69, 2018.

\bibitem{hardt2016equality}
Moritz Hardt, Eric Price, and Nathan Srebro.
\newblock Equality of opportunity in supervised learning.
\newblock In {\em Advances in Neural Information Processing Systems}, pages
  3315--3323, 2016.

\bibitem{ghassami2018fairness}
AmirEmad Ghassami, Sajad Khodadadian, and Negar Kiyavash.
\newblock Fairness in supervised learning: An information theoretic approach.
\newblock In {\em Proceedings of the IEEE International Symposium on
  Information Theory}, pages 176--180, 2018.

\bibitem{kusner2017counterfactual}
Matt~J Kusner, Joshua Loftus, Chris Russell, and Ricardo Silva.
\newblock Counterfactual fairness.
\newblock In {\em Advances in Neural Information Processing Systems}, pages
  4066--4076, 2017.

\bibitem{kilbertus2017avoiding}
Niki Kilbertus, Mateo~Rojas Carulla, Giambattista Parascandolo, Moritz Hardt,
  Dominik Janzing, and Bernhard Sch{\"o}lkopf.
\newblock Avoiding discrimination through causal reasoning.
\newblock In {\em Advances in Neural Information Processing Systems}, pages
  656--666, 2017.

\bibitem{zemel2013learning}
Rich Zemel, Yu~Wu, Kevin Swersky, Toni Pitassi, and Cynthia Dwork.
\newblock Learning fair representations.
\newblock In {\em Proceedings of the International Conference on Machine
  Learning}, pages 325--333, 2013.

\bibitem{menon2018cost}
Aditya~Krishna Menon and Robert~C. Williamson.
\newblock The cost of fairness in binary classification.
\newblock In {\em Proceedings of the Conference on Fairness, Accountability and
  Transparency}, pages 107--118, 2018.

\bibitem{chen2018my}
Irene~Y. Chen, Fredrik~D. Johansson, and David Sontag.
\newblock Why is my classifier discriminatory?
\newblock In {\em Advances in Neural Information Processing Systems}, pages
  3539--3550, 2018.

\bibitem{zhao2019inherent}
Han Zhao and Geoffrey~J Gordon.
\newblock Inherent tradeoffs in learning fair representation.
\newblock {\em arXiv preprint arXiv:1906.08386}, 2019.

\bibitem{friedler2016possibility}
Sorelle~A Friedler, Carlos Scheidegger, and Suresh Venkatasubramanian.
\newblock On the (im)possibility of fairness.
\newblock {\em arXiv preprint arXiv:1609.07236}, 2016.

\bibitem{yeom2018discriminative}
Samuel Yeom and Michael~Carl Tschantz.
\newblock Discriminative but not discriminatory: A comparison of fairness
  definitions under different worldviews.
\newblock {\em arXiv preprint arXiv:1808.08619}, 2018.

\bibitem{NoriegaCamperoBGP2019}
Alejandro Noriega-Campero, Michiel~A. Bakker, Bernardo Garcia-Bulle, and
  Alex~'Sandy' Pentland.
\newblock Active fairness in algorithmic decision making.
\newblock In {\em Proceedings of the AAAI/ACM Conference on Artificial
  Intelligence, Ethics, and Society}, pages 77--83, 2019.

\bibitem{BakkerNTSVP2019}
Michiel~A. Bakker, Alejandro Noriega-Campero, Duy~Patrick Tu, Prasanna
  Sattigeri, Kush~R. Varshney, and Alex~'Sandy' Pentland.
\newblock On fairness in budget-constrained decision making.
\newblock In {\em Proceedings of the KDD Workshop on Explainable Artificial
  Intelligence}, 2019.

\bibitem{garg2019tracking}
Sumegha Garg, Michael~P. Kim, and Omer Reingold.
\newblock Tracking and improving information in the service of fairness.
\newblock In {\em Proceedings of the ACM Conference on Economics and
  Computation}, pages 809--824, 2019.

\bibitem{sabato2020bounding}
Sivan Sabato and Elad Yom-Tov.
\newblock Bounding the fairness and accuracy of classifiers from population
  statistics.
\newblock In {\em International Conference on Machine Learning}, pages
  8316--8325. PMLR, 2020.

\bibitem{kim2020model}
Joon~Sik Kim, Jiahao Chen, and Ameet Talwalkar.
\newblock Model-agnostic characterization of fairness trade-offs.
\newblock {\em arXiv preprint arXiv:2004.03424}, 2020.

\bibitem{blum2019recovering}
Avrim Blum and Kevin Stangl.
\newblock Recovering from biased data: Can fairness constraints improve
  accuracy?, 2019.

\bibitem{wick2019unlocking}
Michael Wick, Swetasudha Panda, and Jean-Baptiste Tristan.
\newblock Unlocking fairness: a trade-off revisited.
\newblock In {\em Advances in Neural Information Processing Systems}, pages
  8780--8789, 2019.

\bibitem{sharmadata}
Shubham Sharma, Yunfeng Zhang, Jes{\'u}s M~R{\'\i}os Aliaga, Djallel
  Bouneffouf, Vinod Muthusamy, and Kush~R Varshney.
\newblock Data augmentation for discrimination prevention and bias
  disambiguation.
\newblock In {\em Proceedings of the AAAI/ACM Conference on Artificial
  Intelligence, Ethics, and Society}, pages 358--364, 2020.

\bibitem{CalmonWVRV2018}
Flavio~P. Calmon, Dennis Wei, Bhanukiran Vinzamuri, Karthikeyan
  Natesan~Ramamurthy, and Kush~R. Varshney.
\newblock Data pre-processing for discrimination prevention:
  Information-theoretic optimization and analysis.
\newblock {\em IEEE Journal of Selected Topics in Signal Processing},
  12(5):1106--1119, 2018.

\bibitem{feldman2015certifying}
Michael Feldman, Sorelle~A Friedler, John Moeller, Carlos Scheidegger, and
  Suresh Venkatasubramanian.
\newblock Certifying and removing disparate impact.
\newblock In {\em Proceedings of the ACM SIGKDD International Conference on
  Knowledge Discovery and Data Mining}, pages 259--268, 2015.

\bibitem{varshney2018interpretability}
Kush~R. Varshney, Prashant Khanduri, Pranay Sharma, Shan Zhang, and Pramod~K.
  Varshney.
\newblock Why interpretability in machine learning? {A}n answer using
  distributed detection and data fusion theory.
\newblock In {\em Proceedings of the ICML Workshop on Human Interpretability in
  Machine Learning}, pages 15--20, 2018.

\bibitem{lee2012generalized}
Yuni Lee and Youngchul Sung.
\newblock Generalized {C}hernoff information for mismatched {B}ayesian
  detection and its application to energy detection.
\newblock {\em IEEE Signal Processing Letters}, 19(11):753--756, 2012.

\bibitem{cover2012elements}
Thomas~M. Cover and Joy~A. Thomas.
\newblock {\em Elements of Information Theory}.
\newblock John Wiley \& Sons, 2012.

\bibitem{gretton2007kernel}
Arthur Gretton, Karsten Borgwardt, Malte Rasch, Bernhard Sch{\"o}lkopf, and
  Alex~J. Smola.
\newblock A kernel method for the two-sample-problem.
\newblock In {\em Advances in Neural Information Processing Systems}, pages
  513--520, 2007.

\bibitem{ravikumar2009sparse}
Pradeep Ravikumar, John Lafferty, Han Liu, and Larry Wasserman.
\newblock Sparse additive models.
\newblock {\em Journal of the Royal Statistical Society: Series B (Statistical
  Methodology)}, 71(5):1009--1030, 2009.

\bibitem{scott2013classification}
Clayton Scott, Gilles Blanchard, and Gregory Handy.
\newblock Classification with asymmetric label noise: Consistency and maximal
  denoising.
\newblock In {\em Proceedings of the Conference On Learning Theory}, pages
  489--511, 2013.

\bibitem{decoupled}
Cynthia Dwork, Nicole Immorlica, Adam~Tauman Kalai, and Max Leiserson.
\newblock Decoupled classifiers for group-fair and efficient machine learning.
\newblock In {\em Proceedings of the Conference on Fairness, Accountability and
  Transparency}, pages 119--133, 2018.

\bibitem{brodersen2010balanced}
Kay~Henning Brodersen, Cheng~Soon Ong, Klaas~Enno Stephan, and Joachim~M
  Buhmann.
\newblock The balanced accuracy and its posterior distribution.
\newblock In {\em Proceedings of the International Conference on Pattern
  Recognition}, pages 3121--3124, 2010.

\bibitem{motwani1995randomized}
Rajeev Motwani and Prabhakar Raghavan.
\newblock {\em Randomized Algorithms}.
\newblock Cambridge University Press, 1995.

\bibitem{berend2015finite}
Daniel Berend and Aryeh Kontorovich.
\newblock A finite sample analysis of the naive {B}ayes classifier.
\newblock {\em Journal of Machine Learning Research}, 16:1519--1545, 2015.

\bibitem{cote2012chernoff}
Fran{\c{c}}ois~D C{\^o}t{\'e}, Ioannis~N Psaromiligkos, and Warren~J Gross.
\newblock A {C}hernoff-type lower bound for the {G}aussian {Q}-function.
\newblock {\em arXiv preprint arXiv:1202.6483}, 2012.

\bibitem{berisha2015empirically}
Visar Berisha, Alan Wisler, Alfred~O Hero, and Andreas Spanias.
\newblock Empirically estimable classification bounds based on a nonparametric
  divergence measure.
\newblock {\em IEEE Transactions on Signal Processing}, 64(3):580--591, 2015.

\bibitem{bhattacharyya1946measure}
Anil Bhattacharyya.
\newblock On a measure of divergence between two multinomial populations.
\newblock {\em Sankhy{\=a}: The Indian Journal of Statistics}, 7(4):401--406,
  1946.

\bibitem{kailath1967divergence}
Thomas Kailath.
\newblock The divergence and {B}hattacharyya distance measures in signal
  selection.
\newblock {\em IEEE Transactions on Communication Technology}, 15(1):52--60,
  1967.

\bibitem{zafar2017fairness}
Muhammad~Bilal Zafar, Isabel Valera, Manuel Gomez~Rodriguez, and Krishna~P
  Gummadi.
\newblock Fairness beyond disparate treatment \& disparate impact: Learning
  classification without disparate mistreatment.
\newblock In {\em Proceedings of the International Conference on World Wide
  Web}, pages 1171--1180, 2017.

\bibitem{donini2018empirical}
Michele Donini, Luca Oneto, Shai Ben-David, John~S Shawe-Taylor, and
  Massimiliano Pontil.
\newblock Empirical risk minimization under fairness constraints.
\newblock In {\em Advances in Neural Information Processing Systems}, pages
  2791--2801, 2018.

\bibitem{celis2019}
L.~Elisa Celis, Lingxiao Huang, Vijay Keswani, and Nisheeth~K. Vishnoi.
\newblock Classification with fairness constraints: A meta-algorithm with
  provable guarantees.
\newblock In {\em Proceedings of the ACM Conference on Fairness,
  Accountability, and Transparency}, pages 319--328, 2019.

\bibitem{kamiran2012data}
Faisal Kamiran and Toon Calders.
\newblock Data preprocessing techniques for classification without
  discrimination.
\newblock {\em Knowledge and Information Systems}, 33(1):1--33, 2012.

\bibitem{jiang2019identifying}
Heinrich Jiang and Ofir Nachum.
\newblock Identifying and correcting label bias in machine learning.
\newblock {\em arXiv preprint arXiv:1901.04966}, 2019.

\bibitem{emelianov2020fair}
Vitalii Emelianov, Nicolas Gast, Krishna~P Gummadi, and Patrick Loiseau.
\newblock On fair selection in the presence of implicit variance.
\newblock In {\em Proceedings of the 21st ACM Conference on Economics and
  Computation}, pages 649--675, 2020.

\bibitem{Chernoff1}
K.~{Sricharan}, D.~{Wei}, and A.~O. {Hero}.
\newblock Ensemble estimators for multivariate entropy estimation.
\newblock {\em IEEE Transactions on Information Theory}, 59(7):4374--4388,
  2013.

\bibitem{Chernoff2}
S.~Y. {Sekeh}, B.~{Oselio}, and A.~O. {Hero}.
\newblock Multi-class bayes error estimation with a global minimal spanning
  tree.
\newblock In {\em 2018 56th Annual Allerton Conference on Communication,
  Control, and Computing (Allerton)}, pages 676--681, 2018.

\bibitem{Chernoff3}
A.~{Wisler}, V.~{Berisha}, D.~{Wei}, K.~{Ramamurthy}, and A.~{Spanias}.
\newblock Empirically-estimable multi-class classification bounds.
\newblock In {\em 2016 IEEE International Conference on Acoustics, Speech and
  Signal Processing (ICASSP)}, pages 2594--2598, 2016.

\bibitem{gallager2012detection}
RG~Gallager.
\newblock Detection, decisions, and hypothesis testing.
\newblock \url{http://web.mit.edu/gallager/www/papers/chap3.pdf}, 2012.

\bibitem{boucheron2013concentration}
St{\'e}phane Boucheron, G{\'a}bor Lugosi, and Pascal Massart.
\newblock {\em Concentration inequalities: A nonasymptotic theory of
  independence}.
\newblock Oxford university press, 2013.

\end{thebibliography}
\bibliographystyle{unsrt}

\clearpage
\onecolumn
\appendix
\section{Background on Chernoff Information}
\label{app:background}

In this section, we provide a brief background on Chernoff bounds and Chernoff information, leading to the derivation of the results under equal priors, i.e., $\pi_0=\pi_1=\frac{1}{2}$. We discuss the case of unequal priors in \Cref{app:unequal}.

Consider a detector of the form $T(x)\geq\tau$ for classification between two hypothesis $H_0:X\sim P_0(x)$ and $H_1:X \sim P_1(x)$. Recall that the log-generating functions for this detector are defined as follows: 
\begin{align}
\Lambda_0(u) =  \log{\mathbb{E}[e^{u T(X)}|H_0] },  \text{  and }
\Lambda_1(u) =  \log{\mathbb{E}[e^{u T(X)}|H_1] }. 
\end{align}

\subsection{Proof of \Cref{lem:chernoff_exponent}}
\label{app:chernoff_bound}
We first state the Chernoff bound (see Chapter 2.2 in \cite{boucheron2013concentration}) here, which is a well-known tight bound for approximating error probabilities. For a random variable $T$,
\begin{equation}
\Pr{(T \geq \tau)} = \Pr{(e^{uT} \geq e^{u\tau})} \leq \frac{\mathbb{E}[e^{uT}]}{e^{u\tau}} \ \ \forall u>0.\label{eq:chernoff_bound}
\end{equation}


\begin{proof}[Proof of \Cref{lem:chernoff_exponent}]
Using the Chernoff bound, we can bound $P_{\mathrm{FP}}^{(T)}(\tau)$ as follows:
\begin{align}
P_{\mathrm{FP}}^{(T)}(\tau) & = \Pr{(T(X) \geq \tau|H_0)}  \leq  \frac{\mathbb{E}[e^{uT(X)}|H_0]}{e^{u\tau}}  = \frac{e^{\Lambda_0(u)}}{e^{u\tau}}\ \ \forall u>0.
\end{align}
Thus,
$
  -\log{P_{\mathrm{FP}}^{(T)}(\tau) }
\geq \sup_{u>0} \left( u\tau- \Lambda_0(u)\right) = E_{\mathrm{FP}}^{(T)}(\tau).
$ Similarly, using the Chernoff bound, we have
\begin{align}
P_{\mathrm{FN}}^{(T)}(\tau) & = \Pr{(T(X) < \tau|H_1)}   \leq  \frac{\mathbb{E}[e^{uT(X)}|H_1]}{e^{u\tau}} = \frac{e^{\Lambda_1(u)}}{e^{u\tau}}\ \ \forall u<0.
\end{align}
Thus,
$
 -\log{P_{\mathrm{FN}}^{(T)}(\tau) } \geq \sup_{u<0} \left( u\tau- \Lambda_1(u)\right)=E_{\mathrm{FN}}^{(T)}(\tau).$
\end{proof}

\subsection{Properties of log-generating functions}
\label{app:properties_log_gen}
Here, we state some useful properties of the log-generating functions that are used later in the other proofs/explanations.

\begin{propty}[Convexity] The log-generating functions $\Lambda_0(u)$ and $\Lambda_1(u)$ are convex in $u$.
\label{propty:convexity}
\end{propty}

\begin{proof}[Proof of \Cref{propty:convexity}] The proof follows directly using H\"{o}lder's inequality. For any $u$ and $v$, and $\alpha\in [0,1]$,
\begin{align}
\mathbb{E}[e^{(\alpha u + (1-\alpha)v)T(X)}|H_0] &= \mathbb{E}[e^{\alpha u T(X)} e^{(1-\alpha)vT(X)}|H_0] \leq \left(\mathbb{E}[|e^{\alpha u T(X)}|^{\frac{1}{\alpha}}|H_0] \right)^{\alpha} \left(\mathbb{E}[|e^{(1-\alpha)vT(X)}|^{\frac{1}{1-\alpha}}|H_0] \right)^{1-\alpha}.
\label{eq:holder}
\end{align}

This leads to,
\begin{align}
 \Lambda_0(\alpha u + (1-\alpha)v) 
= \log{\mathbb{E}[e^{(\alpha u + (1-\alpha)v)T(X)}|H_0] } 
& \leq \alpha \log{\mathbb{E}[e^{u T(X)}|H_0] } + (1-\alpha)\log{\mathbb{E}[e^{v T(X)}|H_0] } \nonumber \\
& = \alpha \Lambda_0(u) + (1-\alpha)\Lambda_0(v).
\end{align}

The proof is similar for $\Lambda_1(u)$.
\end{proof}

\begin{propty}[Zero at origin] The log-generating functions $\Lambda_0(u)$ and $\Lambda_1(u)$ are both $0$ at $u=0$.
\label{propty:zero}
\end{propty}

\begin{proof}[Proof of \Cref{propty:zero}] The proof follows by substituting $u=0$ in the expressions of 
$\Lambda_0(u)$ and $\Lambda_1(u)$.
\end{proof}

Next, we prove some properties for the log-generating functions when the detector is \emph{well-behaved}. In general, when using a detector of the form $T(x)\geq \tau,$ we would expect $T(X)$ to be high when $H_1$ is true, and low when $H_0$ is true. We call a detector \emph{well-behaved} if $\mathbb{E}[T(X)|H_0]<0$ and $ \mathbb{E}[T(X)|H_1]>0$. The next property provides more intuition on what the log-generating functions look like for \emph{well-behaved} detectors. 

\begin{propty}[Log-generating functions of well-behaved detectors]  
Suppose that $\mathbb{E}[T(X)|H_0] <0 $ and $\mathbb{E}[T(X)|H_1] >0 $, and $P_0(x)$ and $P_1(x)$ are non-zero for all $x$. Then, the following holds:
\begin{itemize}[itemsep=0pt, topsep=0pt]
\item $\Lambda_0(u)$ and $\Lambda_1(u)$ are strictly convex.
\item $\Lambda_0(u)> 0$ if $u< 0$. 
$\Lambda_1(u)> 0$ if $u> 0$.
\end{itemize}
\label{propty:well_behaved}
\end{propty}

\begin{proof}[Proof of \Cref{propty:well_behaved}]

The convexity of $\Lambda_0(u)$ is proved in \Cref{propty:convexity}. Now $\Lambda_0(u)$ is strictly convex if, for all distinct reals $u$ and $v$, and $\alpha \in (0,1)$, we have,
$$\Lambda_0(\alpha u + (1-\alpha)v) < \alpha \Lambda_0(u) + (1-\alpha)\Lambda_0(v).$$ For the sake of contradiction, let us assume that there exists $u$ and $v$ with $v>u$ such that, $$\Lambda_0(\alpha u + (1-\alpha)v) = \alpha \Lambda_0(u) + (1-\alpha)\Lambda_0(v).$$ This indicates that H\"{o}lder's inequality holds with exact equality in \eqref{eq:holder}, which could happen if and only if $ae^{uT(x)}=be^{vT(x)}$ almost everywhere with respect to the probability measure $P_0(x)$ for constants $a$ and $b$, i.e., $(v-u)T(x)=\log{a/b}$. Thus,
\begin{align}
\mathbb{E}[T(X)|H_0] = \frac{1}{(v-u)}\log{a/b} = \mathbb{E}[T(X)|H_1],
\end{align}
where the last step holds because $P_1(x)$ and $P_0(x)$ are both non-zero everywhere (absolutely continuous with respect to each other). But, this is a contradiction since $\mathbb{E}[T(X)|H_0]< 0 < \mathbb{E}[T(X)|H_1]$. Thus, $\Lambda_0(u)$ is strictly convex. A similar proof can be done for $\Lambda_1(u)$.

For proving the next claim, consider the derivative of $\Lambda_0(u)$.
\begin{align}
\frac{d \Lambda_0(u)}{d u} = \frac{\mathbb{E}[ e^{uT(X)}T(X)|H_0]}{e^{\Lambda_0(u)}}.
\end{align}

The derivative of $\Lambda_0(u)$ at $u=0$ is given by $\mathbb{E}[T(X)|H_0]$ which is strictly less than $0$. Because $\Lambda_0(u)$ is strictly convex in $u$ and $\Lambda_0(0)=0$, if $ \frac{d \Lambda_0(u)}{d u}|_{u=0} <0$, then $\Lambda_0(u)> 0$ for all $u< 0$. 

A similar proof holds for the last claim as well, since the derivative of $\Lambda_1(u)$ at $u=0$ is given by $\mathbb{E}[T(X)|H_1]$ which is strictly greater than $0$, and $\Lambda_1(0)=0$.

\end{proof}

Next, we examine the properties of the log-generating functions for likelihood ratio detectors. Consider the likelihood ratio detector $T_0(x)=\log{\frac{P_1(x)}{P_0(x)}}$. The two conditions $\mathbb{E}[T(X)|H_0]< 0$ and $ \mathbb{E}[T(X)|H_1]>0$ become equivalent to $\mathrm{D}(P_0||P_1)>0$ and $\mathrm{D}(P_1||P_0)>0$ where $\mathrm{D}(\cdot||\cdot)$ denotes the Kullback-Leibler (KL) divergence between the two distributions $P_0(x)$ and $P_1(x)$. Thus, a likelihood ratio detector always satisfies these conditions as long as the KL divergences are well-defined and non-zero. 

\begin{propty}(Log-generating functions of likelihood ratio detectors) Let
$T_0(x)=\log{\frac{P_1(x)}{P_0(x)}}$, and $P_0(x)$ and $P_1(x)$ be non-zero for all $x$ with $\mathrm{D}(P_0||P_1)$ and $\mathrm{D}(P_1||P_0)$ strictly greater than $0$. Then, the following properties hold:
\begin{itemize}[itemsep=0pt, topsep=0pt]
\item  $\Lambda_0(u)$ is $0$ at $u=0$ and $1$, and
 $\Lambda_1(u)$ is $0$ at $u=0$ and $-1$. 
\item $\Lambda_1(u)= \Lambda_0(u+1).$
\item $\mathrm{C}(P_0,P_1)>0.$
\item $\Lambda_0(u)$ and $\Lambda_1(u)$ are continuous, differentiable and strictly convex.
\item The derivatives of $\Lambda_0(u)$ and $\Lambda_1(u)$ are continuous, monotonically increasing and take all values between $-\infty$ and $\infty$.
\item $\Lambda_0(u)$ attains its global minima for $u$ in $(0,1)$.
\item $\Lambda_1(u)$ attains its global minima for $u$ in $(-1,0)$.
\end{itemize}
\label{propty:likelihood}
\end{propty}

We first introduce the arithmetic mean-geometric mean (AM-GM) inequality.
\begin{lem}[AM-GM inequality]
The following inequality is satisfied for $u \in (0,1)$ and $a,b \geq 0$:
\begin{equation}
a^{1-u}b^{u} \leq (1-u)a+ub,
\end{equation} 
where the equality holds if and only if $a=b$.\label{lem:AM-GM}
\end{lem}

\begin{proof}[Proof of \Cref{propty:likelihood}]
The first claim can be verified by direct substitution.

To show that $\Lambda_1(u)= \Lambda_0(u+1),$ observe that,
\begin{align*}
&\Lambda_1(u)  = - \log{\sum_x P_1(x)^{1+u} P_0(x)^u}  = -\log{\sum_x P_1(x)^{1+u} P_0(x)^{1-(1+u)}} =\Lambda_0(u+1).
\end{align*}

Next, we will show that $\mathrm{C}(P_0,P_1)>0$. Observe that, $\mathrm{C}(P_0,P_1)=-\log{\sum_x P_0(x)^{1-u^*} P_1(x)^{u^*}}$ for some $u^*\in (0,1)$. Now, there is at least one $x'$ with $P_0(x')>0$ and $P_1(x')>0$ such that $P_0(x')\neq P_1(x')$ because $\mathrm{D}(P_0||P_1)>0$ and $\mathrm{D}(P_1||P_0)>0$. This leads to a strict AM-GM inequality (\Cref{lem:AM-GM}) as follows:
\begin{align*}
P_0(x')^{1-u^*} P_1(x')^{u^*} < (1-u^*)P_0(x') + u^*P_1(x').
\end{align*}

For all other $x\neq x'$, 
\begin{align*}
P_0(x)^{1-u^*} P_1(x)^{u^*} \leq (1-u^*)P_0(x) + u^*P_1(x).
\end{align*}
Thus, 
\begin{align}
&\sum_x P_0(x)^{1-u^*} P_1(x)^{u^*} {<} \sum_x (1-u^*)P_0(x) + u^*P_1(x) =1 \nonumber \\
& \implies -\log{\sum_x P_0(x)^{1-u^*} P_1(x)^{u^*}} >0.
\end{align}
Thus, $\mathrm{C}(P_0,P_1)>0$. A similar proof extends for continuous distributions as well where the strict inequality holds at least over a set of $x'$s that is not measure $0$.

We move on to the next claim. Since both $P_0(x)$ and $P_1(x)$ are strictly greater than $0$ for all $x$, we have $ P_0(x)^{1-u} P_1(x)^{u}$ to be well-defined and continuous for all values of $u$, including $u=0$ and $u=1$. Thus, $\Lambda_0(u)$ is continuous over the range $(-\infty,\infty)$. 

The derivative of $\Lambda_0(u)$ is given by:
\begin{align}
\frac{d\Lambda_0(u)}{d u} = \frac{\sum_x  P_0(x)^{1-u} P_1(x)^{u}\log{\frac{P_1(x)}{P_0(x)}} }{e^{\Lambda_0(u)}}, \label{eq:derivative}
\end{align}
which is well-defined for all values of $u$.

The strict convexity of $\Lambda_0(u)$ can be proved using \Cref{propty:well_behaved}, because the two conditions $\mathbb{E}[T(X)|H_0]< 0$ and $ \mathbb{E}[T(X)|H_1]>0$ become equivalent to $\mathrm{D}(P_0||P_1)>0$ and $\mathrm{D}(P_1||P_0)>0$. A similar proof extends to $\Lambda_1(u)$.

Now, we move on to the next claim. Observe from \eqref{eq:derivative} that, the derivative is also continuous for all values of $u$ since both $P_0(x)$ and $P_1(x)$ are strictly greater than $0$ for all $x$. It is monotonically increasing because $\Lambda_0(u)$ is strictly convex. Also note that, as $u\to -\infty$, its derivative tends to $-\infty$. Similarly, as $u\to \infty$, its derivative tends to $\infty$. A similar proof extends to $\Lambda_1(u)$.

Lastly, because $\Lambda_0(u)$ is $0$ at $u=0$ and $u=1$, and is a continuous and strictly convex function, it attains its minima for $u$ in $(0,1)$. A similar proof extends to $\Lambda_1(u)$, validating the last claim as well.
\end{proof}

\begin{propty}[Connection to FL transforms] For well-behaved detectors, the following properties hold:
\begin{itemize}[itemsep=0pt, topsep=0pt]
\item If $\tau < \mathbb{E}[T(X)|H_1]$, then
$ \sup_{u<0} \left( u\tau- \Lambda_1(u)\right)  = \sup_{u \in \mathbb{R}} \left( u\tau- \Lambda_1(u)\right).$
\item If $\tau > \mathbb{E}[T(X)|H_0]$, then 
$\sup_{u>0} \left( u\tau- \Lambda_0(u)\right)  = \sup_{u\in \mathbb{R}} \left( u\tau- \Lambda_0(u)\right).$ 
\end{itemize}\label{propty:fl}
\end{propty}

Before the proof, we introduce a lemma that will be used in the proof.

\begin{lem}[Supporting line of a strictly convex function] For a strictly convex and differentiable function $f(u):\mathcal{R}\rightarrow \mathcal{R},$
$$ u_a \frac{d f(u)}{d u}|_{u=u_a} {-} f(u_a) {=} \sup_{u \in \mathcal{R}} \left(u\frac{d f(u)}{d u}|_{u=u_a} - f(u) \right).$$
\label{lem:supp_line}
\end{lem}

The proof of \Cref{lem:supp_line} holds from the definition of strict convexity.

\begin{proof}[Proof of \Cref{propty:fl}]
In general, $\sup_{u \in \mathcal{R}} \left(u\tau - \Lambda_1(u) \right) \geq \sup_{u <0} \left(u\tau - \Lambda_1(u) \right).$ But, here again,

\begin{align}
\sup_{u \in \mathcal{R}} \left(u\tau - \Lambda_1(u) \right) 
&\overset{(a)}{=}\sup_{u \in \mathcal{R}} \left(u\frac{d \Lambda_1(u)}{d u}|_{u=u_a} - \Lambda_1(u) \right)\overset{(b)}{=} u_a \frac{d \Lambda_1(u)}{d u}|_{u=u_a} {-} \Lambda_1(u_a) \nonumber\\
& \overset{(c)}{\leq} \sup_{u <0} \left(u\frac{d \Lambda_1(u)}{d u}|_{u=u_a}  - \Lambda_1(u) \right)  \overset{(d)}{=}\sup_{u <0} \left(u\tau - \Lambda_1(u) \right).
\end{align}
Here (a) holds because the derivative of $\Lambda_1(u)$ is continuous, monotonically increasing and takes all values from $(-\infty,\infty)$ (see \Cref{propty:likelihood}). Thus, for any $\tau$, there exists a single $u_a$ such that $\frac{d \Lambda_1(u)}{d u}|_{u=u_a} =\tau$. Next, (b) holds from \Cref{lem:supp_line}, whereas (c) holds because $\frac{d \Lambda_1(u)}{d u}|_{u=u_a} =\tau < \mathbb{E}[T(X)|H_1]=\frac{d \Lambda_1(u)}{d u}|_{u=0}$ and the derivative is monotonically increasing (see \Cref{propty:likelihood}) implying $u_a<0$. Lastly (d) holds by again substituting $\tau=\frac{d \Lambda_1(u)}{d u}|_{u=u_a}$. This proves the first claim.

Similarly, in general, we have $\sup_{u \in \mathcal{R}} \left(u\tau - \Lambda_0(u) \right) \geq \sup_{u >0} \left(u\tau - \Lambda_0(u) \right).$ But, here again,
\begin{align}
\sup_{u \in \mathcal{R}} \left(u\tau - \Lambda_0(u) \right) 
&\overset{(a)}{=}\sup_{u \in \mathcal{R}} \left(u\frac{d \Lambda_0(u)}{d u}|_{u=u_a} - \Lambda_0(u) \right)\overset{(b)}{=} u_a \frac{d \Lambda_0(u)}{d u}|_{u=u_a} {-} \Lambda_0(u_a) \nonumber\\
& \overset{(c)}{\leq} \sup_{u >0} \left(u\frac{d \Lambda_0(u)}{d u}|_{u=u_a}  - \Lambda_0(u) \right)  \overset{(d)}{=}\sup_{u >0} \left(u\tau - \Lambda_0(u) \right).
\end{align}
Here (a) holds because the derivative of $\Lambda_0(u)$ is continuous, monotonically increasing and takes all values from $(-\infty,\infty)$ (see \Cref{propty:likelihood}). Thus, for any $\tau$, there exists a single $u_a$ such that $\frac{d \Lambda_0(u)}{d u}|_{u=u_a} =\tau$. Next, (b) holds from \Cref{lem:supp_line}, whereas (c) holds because $\frac{d \Lambda_0(u)}{d u}|_{u=u_a} =\tau > \mathbb{E}[T(X)|H_0]=\frac{d \Lambda_0(u)}{d u}|_{u=0}$ and the derivative is monotonically increasing (see \Cref{propty:likelihood}) implying $u_a>0$. Lastly (d) holds by again substituting $\tau=\frac{d \Lambda_0(u)}{d u}|_{u=u_a}$.

\end{proof}

\subsection{Log Generating Functions for Gaussians}
\label{app:example}

Let $P_0(x)\sim\mathcal{N}({\mu_0},\sigma^2\mathbf{I})$ and $P_1(x)\sim\mathcal{N}(\mu_1,\sigma^2\mathbf{I})$, where $\mu_0$ and $\mu_1$ are vectors and $\mathbf{I}$ is an identity matrix. We derive the log-generating functions for likelihood ratio detectors corresponding to these two distributions.

\begin{align}
\Lambda_0(u)=\log{\int P_1(x)^uP_0(x)^{1-u} dx } 
& = \log{ \int e^{\frac{-u}{2\sigma^2}((x-\mu_1)^T(x-\mu_1)-(x-\mu_0)^T(x-\mu_0))}  P_0(x) dx } \nonumber \\
& = \log{  e^{\frac{-u}{2\sigma^2}(\mu_1^T\mu_1-\mu_0^T\mu_0)} \int e^{\frac{-u}{2\sigma^2}(-2x^T(\mu_1-\mu_0)))}  P_0(x) dx } \nonumber    \\
& \overset{(a)}{=} \log{  e^{\frac{-u}{2\sigma^2}(\mu_1^T\mu_1-\mu_0^T\mu_0)} 
e^{\frac{-u}{2\sigma^2}(-2\mu_0^T(\mu_1-\mu_0)))} e^{\frac{u^2}{2\sigma^2}(||\mu_1-\mu_0||^2_2))}   } \nonumber    \\
& = \log{  e^{\frac{-u}{2\sigma^2}(||\mu_1-\mu_0||^2_2)} 
 e^{\frac{u^2}{2\sigma^2}(||\mu_1-\mu_0||^2_2))}   } \nonumber \\
& = \frac{1}{2\sigma^2}||\mu_1-\mu_0||_2^2 u(u-1),
\end{align}
where (a) is derived using the expression of the moment generating function of a Gaussian distribution.

\subsection{Proof of \Cref{lem:separability}}
\label{app:chernoff_information}

\begin{proof}[Proof of \Cref{lem:separability}]

Under equal priors $\pi_0=\pi_1=\frac{1}{2}$, the detector that minimizes the Bayesian probability of error, i.e., $P_{e,T}(\tau)=\pi_0P_{\mathrm{FP},T}(\tau) + \pi_1P_{\mathrm{FN},T}(\tau)$ is the likelihood ratio detector given by $T(x)=\log{\frac{P_1(x)}{P_0(x)}}\geq0$ (for $\pi_0=\pi_1=\frac{1}{2}$). The proof is available in Theorem 3.1 of \cite{gallager2012detection}. 

Here, we will show that the Chernoff exponent of the probability of error for this detector, i.e., $E_{e,T}(0)$ is equal to $\mathrm{C}(P_0,P_1)=-\min_{u \in (0,1)} \log{\sum_{x}P_0(x)^{(1-u)}P_1(x)^u}$. 

Note that, 
\begin{align}
E_{\mathrm{FP},T}(0) & = \sup_{u>0}-\Lambda_0(u)  = -\min_{u \in (0,1)} \log{\sum_{x}P_0(x)^{(1-u)}P_1(x)^u},
\end{align}
where the last step follows because $\Lambda_0(u)$ attains its minima in the range $u \in (0,1)$ (see \Cref{propty:likelihood}).
\begin{align}
E_{\mathrm{FN},T}(0) = \sup_{u<0} - \Lambda_1(u) &\overset{(a)}{=} - \min_{u \in (-1,0)} \log{\sum_{x}P_0(x)^{(-u)}P_1(x)^{(1+u)}}\nonumber \\
& = -\min_{u'=u+1 \in (0,1)} \log{\sum_{x}P_0(x)^{(1-u')}P_1(x)^{(u')},}
\end{align}
where (a) also holds because $\Lambda_1(u)$ attains its minima in the range $u \in (-1,0)$ (see \Cref{propty:likelihood}). Lastly, 
\begin{align}
E_{e,T}(0) = \min \{E_{\mathrm{FP},T}(0),E_{\mathrm{FN},T}(0) \} = \mathrm{C}(P_0,P_1).
\end{align}
\end{proof}

\section{Appendix to \Cref{subsec:limit}}

Before the proofs, we introduce a lemma that will be used in the proofs.

\begin{lem} Let $P_0(x)$ and $P_1(x)$ be non-zero for all $x$ and $\mathrm{D}(P_0||P_1)$ and $\mathrm{D}(P_1||P_0)$ be strictly greater than $0$. For likelihood ratio detectors of the form $T_0(x)=\log{\frac{P_1(x)}{P_0(x)}}\geq \tau_0$, if $\tau_0 \neq 0$, then one of the following statements is true:
 $$E_{\mathrm{FN},T_0}(\tau_0) < \mathrm{C}(P_0,P_1) < E_{\mathrm{FP},T_0}(\tau_0), \text{or  }   E_{\mathrm{FP},T_0}(\tau_0) < \mathrm{C}(P_0,P_1) < E_{\mathrm{FN},T_0}(\tau_0) .$$ 
\label{lem:likelihood}
\end{lem}
\begin{proof}[Proof of \Cref{lem:likelihood}]

Let us analyze the scenario where $\tau_0>0$. Observe that,
\begin{align}
E_{\mathrm{FP}, T_0}(\tau_0)  = \sup_{u>0} (u\tau_0-\Lambda_0(u)) 
& \geq u^*_0\tau_0 -\Lambda_0(u^*_0) && [\text{for any }u_0^*>0]\nonumber \\
& > -\Lambda_0(u_0^*) && [\text{since }u_0^*\tau_0>0]\nonumber \\
&\overset{(a)}{=} \mathrm{C}(P_0,P_1),
\end{align}
where (a) follows if we choose $u_0^*= \arg \min \Lambda_0(u)$ (from \Cref{propty:likelihood}, $\Lambda_0(u)$ attains its minima for some $u \in (0,1)$) and  $\Lambda_0(u_0^*)=-\mathrm{C}(P_0,P_1)$ (by definition).  

Now, we will show that  $E_{\mathrm{FN}, T_0}(\tau_0){<} \mathrm{C}(P_0,P_1)$ when $\tau_0>0$.

\textbf{Case 1:} $\tau_0 \geq \frac{d \Lambda_1(u)}{du}|_{u=0}=\mathrm{D}(P_1||P_0)$
\begin{align}
E_{\mathrm{FN}, T_0}(\tau_0)=\sup_{u<0} (u\tau_0-\Lambda_1(u)) 
& \leq \sup_{u<0} (u \mathrm{D}(P_1||P_0)-\Lambda_1(u)) \ [\text{since }\tau_0 \geq \mathrm{D}(P_1||P_0)]\nonumber \\
& \leq \sup_{u\in \mathcal{R}} (u \mathrm{D}(P_1||P_0)-\Lambda_1(u)) \nonumber \\
&\overset{(a)}{=} (0\cdot \mathrm{D}(P_1||P_0)-\Lambda_1(0)) \overset{(b)}{=} 0 \overset{(c)}{<} \mathrm{C}(P_0,P_1),
\end{align}
where (a) holds from \Cref{lem:supp_line} because $\frac{d \Lambda_1(u)}{du}|_{u=0}=\mathrm{D}(P_1||P_0)$, and (b) and (c) hold from \Cref{propty:likelihood} since $\Lambda_1(0)=0$ and $\mathrm{C}(P_0,P_1)>0$.

\textbf{Case 2:} $0 < \tau_0 <  \frac{d \Lambda_1(u)}{du}|_{u=0}=\mathrm{D}(P_1||P_0)$
\begin{align}
 E_{\mathrm{FN}, T_0}(\tau_0)  = \sup_{u<0} (u\tau_0-\Lambda_1(u)) &\leq \sup_{u\in \mathcal{R}} (u\tau_0-\Lambda_1(u)) \nonumber \\
&\overset{(a)}{=} \sup_{u\in \mathcal{R}}(u \tau_0 -\Lambda_1(u)) \ \ [\text{where $\frac{d \Lambda_1(u)}{du}|_{u=u_a}= \tau_0$}]\nonumber \\
& \overset{(b)}{=}u_a\tau_0 -\Lambda_1(u_a) \ \ \nonumber \\
& \overset{(c)}{<} -\Lambda_1(u_a) \ \ \ \ \ [\text{since  } u_a\tau_0 < 0] \nonumber \\
& \leq -\min_{u}\Lambda_1(u)
\nonumber \\
& \overset{(d)}{=}-\min_{u\in (-1,0)}\Lambda_1(u) =\mathrm{C}(P_0,P_1)
\end{align}
Here, (a) holds because the derivative of $\Lambda_1(u)$ is continuous, monotonically increasing and takes all values from $-\infty$ to $\infty$ (see \Cref{propty:likelihood}). Thus, for any $\tau_0$, there exists a single $u_a$ such that $\frac{d \Lambda_1(u)}{du}|_{u=u_a}= \tau_0$. Next, (b) holds from \Cref{lem:supp_line}, (c) holds because $\frac{d \Lambda_1(u)}{du}|_{u=u_a}= \tau_0 < \frac{d \Lambda_1(u)}{du}|_{u=0}$, and the derivative is monotonically increasing, implying $u_a<0$. Lastly (d) holds because $\Lambda_1(u)$ attains its minima in the range $u \in (-1,0)$ (see \Cref{propty:likelihood}).

Thus, for $\tau_0>0$, we get $E_{\mathrm{FN}, T_0}(\tau_0) < \mathrm{C}(P_0,P_1) < E_{\mathrm{FP}, T_0}(\tau_0) .$

The proof is similar for the scenario where $\tau_0<0$, and leads to $E_{\mathrm{FP},T_0}(\tau_0) < \mathrm{C}(P_0,P_1) < E_{\mathrm{FN},T_0}(\tau_0).$  

\end{proof}

\subsection{Proof of \Cref{lem:trade-off}}

\begin{proof}[Proof of \Cref{lem:trade-off}]
Suppose there exists two likelihood ratio detectors for the two groups such that, $E_{\mathrm{FN},T_0}(\tau_0)=E_{\mathrm{FN},T_1}(\tau_1)$. Since $\mathrm{C}(P_0,P_1)<\mathrm{C}(Q_0,Q_1)$, \emph{at most} one of the two exponents $E_{\mathrm{FN},T_0}(\tau_0)$ and $E_{\mathrm{FN},T_1}(\tau_1)$ can be equal to their corresponding Chernoff information $\mathrm{C}(P_0,P_1)$ or $\mathrm{C}(Q_0,Q_1)$. Without loss of generality, we may assume that $E_{\mathrm{FN},T_0}(\tau_0)\neq \mathrm{C}(P_0,P_1)$. This implies that $\tau_0 \neq 0$ because in the proof of \Cref{lem:separability}, we already showed that when  $\tau_0=0$, we always have $E_{\mathrm{FN},T_0}(0)=E_{\mathrm{FP},T_0}(0)=\mathrm{C}(P_0,P_1).$ Since $\tau_0 \neq 0$, using \Cref{lem:likelihood}, we either have $E_{\mathrm{FN},T_0}(\tau_0) < \mathrm{C}(P_0,P_1) < E_{\mathrm{FP},T_0}(\tau_0) $ or $E_{\mathrm{FP},T_0}(\tau_0) < \mathrm{C}(P_0,P_1) < E_{\mathrm{FN},T_0}(\tau_0) .$ Thus,
\begin{align}
E_{e,T_0}(\tau_0) = \min \{E_{\mathrm{FP},T_0}(\tau_0),E_{\mathrm{FN},T_0}(\tau_0)  \} {<} \mathrm{C}(P_0,P_1).
\end{align}
\end{proof}

\subsection{Proof of \Cref{thm:separability}}
\label{app:fundamental_limits}

\begin{proof}[Proof of \Cref{thm:separability}]
The first claim follows directly from \Cref{lem:separability} by choosing the likelihood ratio detectors for the two groups with thresholds $\tau_0=\tau_1=0$, i.e., the Bayes optimal detector under equal priors.

Now, we prove the second claim. Suppose that we choose the Bayes optimal classifiers $T_0(x)\geq \tau_0$ and $T_1(x)\geq \tau_1$ for the two groups. Then, we have $E_{\mathrm{FN},T_0}(\tau_0)=C(P_0,P_1)$ and $E_{\mathrm{FN},T_1}(\tau_1)=C(Q_0,Q_1)$ which are not equal. Thus, $|E_{\mathrm{FN},T_0}(\tau_0)-E_{\mathrm{FN},T_1}(\tau_1)|\neq 0$. 

Assume (for the sake of contradiction) that there is a likelihood ratio detector such that $E_{e,T_0}(\tau_0)>\mathrm{C}(P_0,P_1)$. 

Now, if $\tau_0=0$, then we have $E_{e,T_0}(\tau_0) = \mathrm{C}(P_0,P_1) $ (from \Cref{lem:separability}). Alternately, if $\tau_0 \neq 0,$  then we either have  $E_{\mathrm{FN}, T_0}(\tau_0) < \mathrm{C}(P_0,P_1) < E_{\mathrm{FP}, T_0}(\tau_0)$ or  $E_{\mathrm{FP},T_0}(\tau_0) < \mathrm{C}(P_0,P_1) < E_{\mathrm{FN},T_0}(\tau_0)$ (from \Cref{lem:likelihood}). Thus, 
\begin{align}
E_{e,T_0}(\tau_0)=\min \{E_{\mathrm{FP},T_0}(\tau_0), E_{\mathrm{FN},T_0}(\tau_0) \} < \mathrm{C}(P_0,P_1).
\end{align}
For both cases, we have a contradiction, implying that $E_{e,T_0}(\tau_0)\leq \mathrm{C}(P_0,P_1)< \mathrm{C}(Q_0,Q_1)$ for all likelihood ratio detectors.
\end{proof}






\subsection{Proofs of \Cref{lem:both_group} and \Cref{lem:one_group}}
\label{app:group_cases}

\begin{proof}[Proof of \Cref{lem:both_group}]
Let $\tau_0^*=0$. Using \Cref{lem:separability}, this ensures, $$E_{\mathrm{FN},T_0}(0)=E_{\mathrm{FP},T_0}(0)=\mathrm{C}(P_0,P_1).$$
Now, we will show that the only value of $\tau_1^*$ that will satisfy $E_{\mathrm{FN},T_1}(\tau_1^*)=E_{\mathrm{FN},T_0}(0)$ is a $\tau_1^*{>}0$ such that $E_{\mathrm{FN},T_1}(\tau_1^*)=\mathrm{C}(P_0,P_1).$ To prove that such a $\tau_1^*$ exists, consider the function: $$g(u)=u\frac{d \Lambda_1(u)}{d (u)}-\Lambda_1(u),$$
where $\Lambda_1(u)$ is the log-generating transform for $z=1.$
The function $g(u)$ is continuous. 
At $u=0$, $g(u)=0$ and at $u=u_1^*$ (where $u_1^*=\arg \min \Lambda_1(u)$ and lies in $(-1,0)$ from \Cref{propty:likelihood}) we have $g(u)=\mathrm{C}(Q_0,Q_1)$. Because $g(u)$ is continuous, there exists a $u_a \in (u_1^*,0)$ such that $g(u_a)=\mathrm{C}(P_0,P_1)$ which lies between $0$ and $\mathrm{C}(Q_0,Q_1)$. If we set $\tau_1^*=\frac{d \Lambda_1(u)}{d (u)}|_{u=u_a},$ we have, $$\mathrm{C}(P_0,P_1)=g(u_a)\overset{\text{\Cref{lem:supp_line}}}{=}\sup_{u\in \mathcal{R}}(u \tau_1^* -\Lambda_1(u)).$$ 

Now, in general, $\sup_{u<0}(u \tau_1^* -\Lambda_1(u)) \leq \sup_{u\in \mathcal{R}}(u \tau_1^* -\Lambda_1(u)) =g(u_a)$. But again, $\sup_{u<0}(u \tau_1^* -\Lambda_1(u))\geq u_a \tau_1^* -\Lambda_1(u_a)=g(u_a)$ since $u_a \in (u_1^*,0)$. Thus, 
$$E_{\mathrm{FN},T_0}(\tau_1^*)=\sup_{u<0}(u \tau_1^* -\Lambda_1(u))=g(u_a)=\mathrm{C}(P_0,P_1).$$ 

Also note that $\tau_1^*>0$ because the derivative of $\Lambda_1(u)$ is monotonically increasing and $u_a>u_1^*$, leading to $ \tau_1^*=\frac{d \Lambda_1(u)}{d (u)}|_{u=u_a}>\frac{d \Lambda_1(u)}{d (u)}|_{u=u_1^*} =0$.

Now that we have a $\tau_1^*$ such that $E_{\mathrm{FN},T_1}(\tau_1^*)=\mathrm{C}(P_0,P_1)$ which is strictly less that $\mathrm{C}(Q_0,Q_1)$, we must have $E_{\mathrm{FP},T_1}(\tau_1^*)> \mathrm{C}(Q_0,Q_1)$ (from \Cref{lem:likelihood}).

This leads to, $$\min\{ E_{\mathrm{FP},T_0}(0),E_{\mathrm{FN},T_0}(0), E_{\mathrm{FP},T_1}(\tau^*_1),E_{\mathrm{FN},T_1}(\tau^*_1)\}{=} \mathrm{C}(P_0,P_1).$$ 

For any other choice of $\tau_0^*\neq0$, we either have $E_{\mathrm{FP},T_0}(\tau_0^*)< \mathrm{C}(P_0,P_1)< E_{\mathrm{FN},T_0}(\tau_0^*),$ or $E_{\mathrm{FN},T_0}(\tau_0^*)< \mathrm{C}(P_0,P_1)< E_{\mathrm{FP},T_0}(\tau_0^*),$ implying $$\min\{ E_{\mathrm{FP},T_0}(\tau_0^*),E_{\mathrm{FN},T_0}(\tau_0^*), E_{\mathrm{FP},T_1}(\tau^*_1),E_{\mathrm{FN},T_1}(\tau^*_1)\}{<} \mathrm{C}(P_0,P_1).$$ 
\end{proof}

\begin{proof}[Proof of \Cref{lem:one_group}]
We are given that,
$$E_{\mathrm{FN},T_1}(\tau_1)=E_{\mathrm{FP},T_1}(\tau_1)=\mathrm{C}(Q_0,Q_1).$$
Now, we will show that the only value of $\tau_0^*$ that will satisfy $E_{\mathrm{FN},T_0}(\tau_0^*)=\mathrm{C}(Q_0,Q_1)$ is a $\tau_0^*<0.$  To prove that such a $\tau_0^*$ exists, consider the function $$g(u)=u\frac{d \Lambda_1(u)}{d (u)}-\Lambda_1(u),$$
where $\Lambda_1(u)$ is the log-generating transform for $z=0.$
The function $g(u)$ is continuous. 
At $u=u_1^*$ (where $u_1^*=\arg \min \Lambda_1(u)$ and lies in $(-1,0)$ from \Cref{propty:likelihood}), we have $g(u_1^*)=\mathrm{C}(P_0,P_1)$ and as $u\to -\infty$, we have $g(u) \to \infty$. Because $g(u)$ is continuous, there exists a $u_a \in (-\infty,u_1^*)$ such that $g(u_a)=\mathrm{C}(Q_0,Q_1)$ which lies between $\mathrm{C}(P_0,P_1)$ and $\infty$. If we set $\tau_0^*=\frac{d \Lambda_1(u)}{d (u)}|_{u=u_a},$ we have, $$\mathrm{C}(Q_0,Q_1)=g(u_a)\overset{ \text{\Cref{lem:supp_line}}}{=}\sup_{u\in \mathcal{R}}(u \tau_0^* -\Lambda_1(u)).$$ 

Now, in general, $\sup_{u<0}(u \tau_0^* -\Lambda_1(u)) \leq \sup_{u\in \mathcal{R}}(u \tau_0^* -\Lambda_1(u))= g(u_a)$. But again, $\sup_{u<0}(u \tau_0^* -\Lambda_1(u)) \geq u_a \tau_0^* -\Lambda_1(u_a)=g(u_a)$ since $u_a<u_1^*<0.$ Thus,
$$E_{\mathrm{FN},T_0}(\tau_0^*)=\sup_{u<0}(u \tau_0^* -\Lambda_1(u))=g(u_a)=\mathrm{C}(Q_0,Q_1).$$ 

This $\tau_0^*$ is less than $0$ because the derivative of $\Lambda_1(u)$ is monotonically increasing and $u_a<u_1^*$, leading to $\tau_0^*= \frac{\Lambda_1(u)}{d (u)}|_{u=u_a}< \frac{\Lambda_1(u)}{d (u)}|_{u=u_1^*}=0$.

Now that we have a $\tau_0^*$ such that $E_{\mathrm{FN},T_0}(\tau_0^*)=\mathrm{C}(Q_0,Q_1)$ which is strictly greater that $\mathrm{C}(P_0,P_1)$, we must have $E_{\mathrm{FP},T_0}(\tau_0^*)< \mathrm{C}(P_0,P_1)$ (from \Cref{lem:likelihood}).

This leads to, $$\min\{ E_{\mathrm{FP},T_0}(\tau^*_0),E_{\mathrm{FN},T_0}(\tau_0^*)\} <\mathrm{C}(P_0,P_1).$$ 

\end{proof}

\section{Appendix to Section~\ref{subsec:accord}}
\label{app:mismatched}

\begin{proof}[Proof of \Cref{thm:feasibility}]
From \Cref{lem:one_group}, there exists a likelihood ratio detector of the form $T_0(x)=\log{\frac{P_1(x)}{P_0(x)}}\geq\tau_0^*$ such that 
\begin{equation}
E_{\mathrm{FN},T_0}(\tau_0^*)=\mathrm{C}(Q_0,Q_1). \label{eq:exponent}
\end{equation}
In the proof of \Cref{lem:one_group}, we showed that this $\tau_0^*<0.$

Now, we will show that there exists $\widetilde{P}_0(x)$ and $\widetilde{P}_1(x)$ such that their optimal detector $\widetilde{T_0}(x)=\log{\frac{\widetilde{P}_1(x)}{\widetilde{P}_0(x)}}\geq0$ is equivalent to the detector $T_0(x)\geq\tau_0^*.$

Let $\widetilde{P}_0(x)=\frac{P_0(x)^{(1-w)}P_1(x)^w }{\sum_x P_0(x)^{(1-w)}P_1(x)^w }$  and $\widetilde{P}_1(x)=\frac{P_0(x)^{(1-v)}P_1(x)^v }{\sum_x P_0(x)^{(1-v)}P_1(x)^v }$ for some $w,v\in \mathcal{R}$ with $w\neq v$. Observe that,

\begin{align}
 \widetilde{T_0}(x)=\log{\frac{\widetilde{P}_1(x)}{\widetilde{P}_0(x)}} &= (v-w)\log{\frac{P_1(x)}{P_0(x)}} + \log{\frac{\sum_x P_0(x)^{(1-w)}P_1(x)^w }{\sum_x P_0(x)^{(1-v)}P_1(x)^v }} \nonumber \\
& = (v-w)\log{\frac{P_1(x)}{P_0(x)}} + \Lambda_0(w)-\Lambda_0(v) \nonumber \\
& =(v-w)\left(\log{\frac{P_1(x)}{P_0(x)}} - \frac{\Lambda_0(v)-\Lambda_0(w)}{v-w}  \right).
\end{align}

Because $\Lambda_0(u)$ is strictly convex with its derivative taking all values from $-\infty$ to $\infty$, one can always find a tangent to $\Lambda_0(u)$ that has a slope $\tau_0^*$ at (say) $u=u_a$. Thus, one can always find pairs of points $(w,v)$ on either sides of $u=u_a$ such that $\tau_0^*=\frac{\Lambda_0(v)-\Lambda_0(w)}{v-w},$ which are essentially pairs of points $(w,v)$ at which a straight line with slope $\tau_0^*$ cuts $\Lambda_0(u)$. In particular,  we can fix $v=1$ and always find a $w<0$ such that 
\begin{equation}
\tau_0^*=\frac{\Lambda_0(v)-\Lambda_0(w)}{v-w} = \frac{-\Lambda_0(w)}{1-w}, \label{eq:substitution}
\end{equation}
because $\Lambda_0(u)$ is continuous taking values $0$ at $u=0$ and $u=1$, and takes all values from $(0,\infty)$ in the range $(-\infty,0)$. Thus, the first claim is proved.

Now, we calculate $\mathrm{C}(\widetilde{P}_0,\widetilde{P}_1)$. 
\allowdisplaybreaks
\begin{align}
\mathrm{C}(\widetilde{P}_0,\widetilde{P}_1)
 = \max_{u\in(0,1)} -\log{\sum_x\widetilde{P}_0(x)^{1-u}\widetilde{P}_1(x)^u } 
& \overset{(a)}{=} \max_{u\in \mathcal{R}} -\log{\sum_x\widetilde{P}_0(x)^{1-u}\widetilde{P}_1(x)^u } \nonumber \\
& \overset{(b)}{=} \max_{u\in \mathcal{R}} -\log{\sum_xP_0(x)^{(1-w)(1-u)}P_1(x)^{w(1-u)+u}} + (1-u)\Lambda_0(w) \nonumber \\
& \overset{(c)}{=}\max_{u\in \mathcal{R}} -\log{\sum_xP_0(x)^{(1-w)(1-u)}P_1(x)^{w(1-u)+u}} + (1-u)(w-1)\tau_0^* \nonumber \\
& \overset{(d)}{=}\max_{u\in \mathcal{R}} (1-u)(w-1)\tau_0^*-\Lambda_1((1-u)(w-1)) \nonumber \\
& \overset{(e)}{=} \sup_{u'\in \mathcal{R}} (u'\tau_0^*-\Lambda_1(u')) \ \ [u'=(1-u)(w-1)] \nonumber \\
& \overset{(f)}{=} \sup_{u'<0} (u'\tau_0^*-\Lambda_1(u')) \ \ [u'=(1-u)(w-1)] \nonumber \\
& \overset{(g)}{=} \mathrm{C}(Q_0,Q_1).
\end{align}
Here (a) holds because the log-generating function $-\log{\sum_x\widetilde{P}_0(x)^{1-u}\widetilde{P}_1(x)^u }$ of a likelihood ratio detector attains its global minima at $(0,1)$ (see \Cref{propty:likelihood}) and (b) holds by substituting $\widetilde{P}_0(x)=\frac{P_0(x)^{(1-w)}P_1(x)^w }{\sum_x P_0(x)^{(1-w)}P_1(x)^w }$  and $\widetilde{P}_1(x)=\frac{P_0(x)^{(1-v)}P_1(x)^v }{\sum_x P_0(x)^{(1-v)}P_1(x)^v }$ with $v=1$. Next, (c) holds by using $\tau_0^*=\frac{\Lambda_0(v)-\Lambda_0(w)}{v-w} = \frac{-\Lambda_0(w)}{1-w}$ (see \eqref{eq:substitution}), (d) holds from the definition of $\Lambda_1((1-u)(w-1))$, (e) holds by a change of variable $u'=(1-u)(w-1),$ (f) holds because $\tau_0^*<0 \leq \mathrm{D}(\widetilde{P}_1||\widetilde{P}_0)=\mathbb{E}[\widetilde{T_0}(X)|\widetilde{H_1}] $ and the detector is well-behaved (see \Cref{propty:fl}), and lastly $(g)$ holds because $E_{\mathrm{FN},T_0}(\tau_0^*)=\mathrm{C}(Q_0,Q_1)$ (see \eqref{eq:exponent}).
\end{proof}

\section{Appendix to Section~\ref{subsec:explainability}}
\label{app:explainability}

\subsection{Proof of \Cref{thm:explainability}}

\begin{proof}[Proof of \Cref{thm:explainability}]

We remind the readers that,
\begin{align}
 \frac{W_0(x,x')}{P_0(x)}=\Pr{(X'=x'|X=x,Z=0,Y=0)}, \text{ and }
\frac{W_1(x,x')}{P_1(x)}=\Pr{(X'=x'|X=x,Z=0,Y=1)}.
\end{align}

First, we would like to prove: $I(X';Y|X,Z=0)> 0 \implies \mathrm{C}(W_0,W_1)>\mathrm{C}(P_0,P_1).$

Suppose that $X'$ is not independent of $Y$ given $X$ and $Z=0$, i.e., $I(X';Y|X,Z=0)> 0.$ This implies that there exists at least one $X=x_a$ such that the distributions of $X'|_{X=x_a,Z=0,Y=0}$ and $X'|_{X=x_a,Z=0,Y=1}$ are different. Therefore, there exists at least one pair $(x',x)=(x_a',x_a)$ for which the following AM-GM inequality (\Cref{lem:AM-GM}) holds with strict inequality for all $u \in (0,1)$, i.e,
\begin{align}
& \left(\frac{W_0(x_a,x_a')}{P_0(x_a)}\right)^{1-u}\left(\frac{W_1(x_a,x_a')}{P_1(x_a)}\right)^{u} 
< (1-u)\frac{W_0(x_a,x_a')}{P_0(x_a)}+u \frac{W_1(x_a,x_a')}{P_1(x_a)}.\label{eq:1}
\end{align}
For all other $(x',x)\neq (x_a',x_a)$, we have (from the AM-GM inequality in \Cref{lem:AM-GM}):
\begin{align}
 \left(\frac{W_0(x,x')}{P_0(x)}\right)^{1-u}\left(\frac{W_1(x,x')}{P_1(x)}\right)^{u} \leq (1-u)\frac{W_0(x,x')}{P_0(x)}+u \frac{W_1(x,x')}{P_1(x)}. \label{eq:2}
\end{align}

Using \eqref{eq:1} and \eqref{eq:2},
\begin{align}
\sum_{x'}\left(\frac{W_0(x_a,x')}{P_0(x_a)}\right)^{1-u}\left(\frac{W_1(x_a,x')}{P_1(x_a)}\right)^{u} 
& < \sum_{x'}\left((1-u)\frac{W_0(x_a,x')}{P_0(x_a)}+u \frac{W_1(x_a,x')}{P_1(x_a)}\right)=1.
\end{align}

This leads to,
\begin{align}
\sum_{x'}W_0(x_a,x')^{1-u}W_1(x_a,x')^u <  P_0(x_a)^{1-u}P_1(x_a)^u. \label{eq:3}
\end{align}

For all other $x\neq x_a$, we have (using \eqref{eq:2} alone),
\begin{align}
\sum_{x'}\big(\frac{W_0(x,x')}{P_0(x)}\big)^{1-u}\big(\frac{W_1(x,x')}{P_1(x)}\big)^{u}  \leq \sum_{x'}\big((1-u)\frac{W_0(x,x')}{P_0(x)}+u \frac{W_1(x,x')}{P_1(x)}\big)=1,
\end{align}
leading to
\begin{align}
\sum_{x'}W_0(x,x')^{1-u}W_1(x,x')^u \leq  P_0(x)^{1-u}P_1(x)^u. \label{eq:4}
\end{align}

Lastly, using \eqref{eq:3} and \eqref{eq:4},
\begin{align}
\sum_{x}\sum_{x'}W_0(x,x')^{1-u}W_1(x,x')^u < \sum_{x} P_0(x)^{1-u}P_1(x)^u,
\end{align}
leading to the claim:
\begin{align}
\mathrm{C}(W_0,W_1)  =- \min_{u \in (0,1)} \log{\sum_{x}\sum_{x'}W_0(x,x')^{1-u}W_1(x,x')^u} 
& > - \min_{u \in (0,1)} \log{\sum_{x} P_0(x)^{1-u}P_1(x)^u}  = \mathrm{C}(P_0,P_1).
\end{align}

We would now like to prove:\\ $\mathrm{C}(W_0,W_1)>\mathrm{C}(P_0,P_1) \implies I(X';Y|X,Z=0)> 0,$
or, $I(X';Y|X,Z=0){=}0 {\implies} \mathrm{C}(W_0,W_1) {\ngtr} \mathrm{C}(P_0,P_1).$

First note that, from the previous proof, $\mathrm{C}(W_0,W_1)\geq \mathrm{C}(P_0,P_1) $ always holds using the AM-GM inequality. Thus, $\mathrm{C}(W_0,W_1) {\ngtr} \mathrm{C}(P_0,P_1)$ is same as $\mathrm{C}(W_0,W_1) {=} \mathrm{C}(P_0,P_1).$

Suppose that $X'$ is independent of $Y$ given $X$ and $Z=0$, i.e., $I(X';Y|X,Z=0)=0$. This implies that,
\begin{align}
&  \Pr(X'=x'|X,Z=0,Y=0) =\Pr(X'=x'|X,Z=0,Y=1) \ \forall x'\nonumber \\
& \Rightarrow \frac{W_0(x,x')}{P_0(x)}= \frac{W_1(x,x')}{P_1(x)}\ \ \ \ \forall x',x\nonumber \\
& \Rightarrow \sum_{x'}\big(\frac{W_0(x,x')}{P_0(x)}\big)^{1-u}\big(\frac{W_1(x,x')}{P_1(x)}\big)^{u} = 1 \ \forall x \ \nonumber \\
& \Rightarrow \sum_{x}\sum_{x'}W_0(x,x')^{1-u}W_1(x,x')^u = \sum_{x} P_0(x)^{1-u}P_1(x)^u.
\end{align}

This leads to
\begin{align}
\mathrm{C}(W_0,W_1) =- \min_{u \in (0,1)} \log{\sum_{x}\sum_{x'}W_0(x,x')^{1-u}W_1(x,x')^u} 
& = - \min_{u \in (0,1)} \log{\sum_{x} P_0(x)^{1-u}P_1(x)^u} = \mathrm{C}(P_0,P_1).
\end{align}

\end{proof}

\section{Unequal Priors}
\label{app:unequal}

\subsection{Unequal Priors on $Y$ but Equal Priors on $Z$}
\label{app:unequal_y}

When the prior probabilities are unequal, we can write $P_{e,T_z}(\tau_z)$ as: $$
P_{e,T_z}(\tau_z){=}\frac{1}{2}(2\pi_0 P_{\mathrm{FP},T_z}(\tau_z)){+}\frac{1}{2}(2\pi_1P_{\mathrm{FN},T_z}(\tau_z))
,$$ and define the Chernoff exponent of $P_{e,T_z}(\tau_z)$, i.e., ${E_{e,T_z}(\tau_z)}$ more generally as follows:
\begin{align*}\min \{ E_{\mathrm{FP},T_z}(\tau_z){-}\log{2\pi_0}, E_{\mathrm{FN},T_z}(\tau_z){-}\log{2\pi_1} \}. 
\end{align*}

\begin{lem}Let the absolute continuity and distinct hypotheses assumptions of \Cref{sec:preliminaries} hold, and $T_z(x)$ be the likelihood ratio detector for the group $Z=z$. 
Then, the value of $\tau_z$ that maximizes  $E_{e,T_z}(\tau_z)$, i.e., \begin{align*}\max_{\tau_z}\ \min \{  E_{\mathrm{FP},T_z}(\tau_z)-\log{2\pi_0}, E_{\mathrm{FN},T_z}(\tau_z)-\log{2\pi_1} \}, \end{align*} is given by $\tau_z^*=\log{\frac{\pi_0}{\pi_1}},$ which is the same as the value of $\tau_z$ that minimizes $P_{e,T_z}(\tau_z),$ i.e.,
\begin{align*}\min_{\tau_z}  \pi_0 P_{\mathrm{FP},T_z}(\tau_z)+\pi_1P_{\mathrm{FN},T_z}(\tau_z). \end{align*}\label{lem:unequal_priors}  
\end{lem}
This likelihood ratio detector $T_z(x){\geq} \log{\frac{\pi_0}{\pi_1}}$ is the Bayes optimal detector for the group. 

Before we proceed to the proof, we discuss another result. Observe that,
\begin{align}
u\tau_0 -\Lambda_0(u)-\log{2\pi_0} 
 = u(\tau_0-\log{\frac{\pi_0}{\pi_1}}) + u \log{\frac{\pi_0}{\pi_1}} -\Lambda_0(u)-\log{2\pi_0} = u\tau' -\widetilde{\Lambda}_0(u)-\log{2},\label{eq:uneq0}
\end{align}
where $\tau'=\tau_0-\log{\frac{\pi_0}{\pi_1}}$, and $\widetilde{\Lambda}_0(u)=\Lambda_0(u)- u \log{\frac{\pi_0}{\pi_1}} +\log{\pi_0}$. Similarly,
\begin{align}
u\tau_0 -\Lambda_1(u)-\log{2\pi_1}  = u(\tau_0-\log{\frac{\pi_0}{\pi_1}}) + u \log{\frac{\pi_0}{\pi_1}} -\Lambda_1(u)-\log{2\pi_1}  = u\tau' -\widetilde{\Lambda}_1(u)-\log{2},\label{eq:uneq1}
\end{align}
where $\tau'=\tau_0-\log{\frac{\pi_0}{\pi_1}}$, and $\widetilde{\Lambda}_1(u)=\Lambda_1(u)- u \log{\frac{\pi_0}{\pi_1}} +\log{\pi_1}$.

We first derive some properties of $\widetilde{\Lambda}_0(u)$ and $\widetilde{\Lambda}_1(u)$.
\begin{lem} Let $P_0(x)$ and $P_1(x)$ be strictly greater than $0$ everywhere and $\mathrm{D}(P_0||P_1)$ and $\mathrm{D}(P_1||P_0)$ be strictly greater than $0$ and $\pi_0$ and $\pi_1$ lie in $(0,1)$. Then, the following properties hold:
\begin{itemize}[itemsep=0pt, topsep=0pt]
\item $\widetilde{\Lambda}_0(u)$ and $\widetilde{\Lambda}_1(u)$ are continuous, differentiable and strictly convex.
\item The derivatives of $\widetilde{\Lambda}_0(u)$ and $\widetilde{\Lambda}_1(u)$ are continuous, monotonically increasing, and take all values from $-\infty$ to $\infty$.
\item $\widetilde{\Lambda}_1(u)=\widetilde{\Lambda}_0(u+1)$.
\end{itemize} \label{lem:unequal_log_gen}
\end{lem}

\begin{proof}[Proof of \Cref{lem:unequal_log_gen}]
Note that, $\widetilde{\Lambda}_0(u)$ is the sum of $\Lambda_0(u)$ and an affine function $- u \log{\frac{\pi_0}{\pi_1}} +\log{\pi_0}$. Because $\Lambda_0(u)$ is continuous, differentiable and strictly convex (from \Cref{propty:likelihood}), $\widetilde{\Lambda}_0(u)$ also satisfies those properties.
The second claim also holds for the same reason because the derivative of $\Lambda_0(u)$ satisfies all these properties (from \Cref{propty:likelihood}).

Lastly,
\begin{align}
\widetilde{\Lambda}_0(u+1)=\Lambda_0(u+1) - (u+1) \log{\frac{\pi_0}{\pi_1}} +\log{\pi_0} 
& =\Lambda_0(u+1) - u \log{\frac{\pi_0}{\pi_1}} +\log{\pi_1} \nonumber \\
& \overset{(a)}{=} \Lambda_1(u) - u \log{\frac{\pi_0}{\pi_1}} +\log{\pi_1} =\widetilde{\Lambda}_1(u),
\end{align}
where (a) holds because $\Lambda_1(u)=\Lambda_0(u+1)$ from \Cref{propty:likelihood}.

\end{proof}

\begin{proof}[Proof of \Cref{lem:unequal_priors}]
We specifically consider the case where $\pi_0\neq \pi_1$ in this proof because the case of equal priors $\pi_0= \pi_1$ can be proved using \Cref{lem:separability} and \Cref{lem:likelihood}.

Without loss of generality, we assume $\pi_0> \pi_1$. Thus, $\log{\frac{\pi_0}{\pi_1}}>0$.

\textbf{Case 1:} $\frac{d \widetilde{\Lambda}_1(u)}{d u}|_{u=0}= \mathrm{D}(P_1||P_0)-\log{\frac{\pi_0}{\pi_1}}>0.$

Observe that, $\frac{d \widetilde{\Lambda}_1(u)}{d u}|_{u=-1}=-\mathrm{D}(P_0||P_1)-\log{\frac{\pi_0}{\pi_1}} <0$ and 
$\frac{d \widetilde{\Lambda}_1(u)}{d u}|_{u=0}= \mathrm{D}(P_1||P_0)-\log{\frac{\pi_0}{\pi_1}}>0.$ Thus, the strictly convex function $\widetilde{\Lambda}_1(u)$ attains its minima in $(-1,0)$ (using \Cref{lem:unequal_log_gen}). Next, using $\widetilde{\Lambda}_0(u+1)=\widetilde{\Lambda}_1(u)$ (also from \Cref{lem:unequal_log_gen}), we have $\widetilde{\Lambda}_0(u)$ attaining its minima in $(0,1)$.

For $\tau'= 0$ (equivalently $\tau_0=\log{\frac{\pi_0}{\pi_1}}$), we have
\begin{align}
 E_{\mathrm{FP}, T_0}(\log{\frac{\pi_0}{\pi_1}})-\log{2\pi_0} 
 \overset{(a)}{=} \sup_{u>0} (u\cdot0 -\widetilde{\Lambda}_0(u)-\log{2}) 
& \overset{(b)}{=} -\min_{u} \widetilde{\Lambda}_0(u) -\log{2} \nonumber\\
& \overset{(c)}{=} -\min_{u} \widetilde{\Lambda}_1(u) -\log{2} \nonumber \\
& \overset{(d)}{=} \sup_{u<0} (u\cdot0 -\widetilde{\Lambda}_1(u)-\log{2}) \nonumber \\
& \overset{(e)}{=} E_{\mathrm{FN}, T_0}(\log{\frac{\pi_0}{\pi_1}})-\log{2\pi_1}. \label{eq:equal_exp}
\end{align}
Here, (a) holds from \eqref{eq:uneq0}, (b) holds because $\widetilde{\Lambda}_0(u)$ attains its minima in $(0,1)$, (c) holds from $\widetilde{\Lambda}_0(u+1)=\widetilde{\Lambda}_1(u)$ (see \Cref{lem:unequal_log_gen}), (d) holds because $\widetilde{\Lambda}_1(u)$ attains its minima in $(-1,0)$, and (e) holds from \eqref{eq:uneq1}.

Next, we will show that, for any other value of $\tau'\neq 0$ ($\tau_0\neq \log{\frac{\pi_0}{\pi_1}}$), we either have 
\begin{align}
& E_{\mathrm{FP}, T_0}(\tau_0)-\log{2\pi_0}  < E_{\mathrm{FP}, T_0}(\log{\frac{\pi_0}{\pi_1}})-\log{2\pi_0}< E_{\mathrm{FN}, T_0}(\tau_0)-\log{2\pi_1}, \text{ or}\nonumber \\
&E_{\mathrm{FN}, T_0}(\tau_0)-\log{2\pi_1}  < E_{\mathrm{FP}, T_0}(\log{\frac{\pi_0}{\pi_1}})-\log{2\pi_0} < E_{\mathrm{FP}, T_0}(\tau_0)-\log{2\pi_0}.
\end{align}
Let $\tau'>0$. Then,
\begin{align}
 E_{\mathrm{FP}, T_0}(\tau_0)-\log{2\pi_0} 
 \overset{(a)}{=} \sup_{u>0} (u\tau' -\widetilde{\Lambda}_0(u)-\log{2}) 
 \overset{(b)}{\geq} (u_0^*\tau' -\widetilde{\Lambda}_0(u_0^*)-\log{2}) 
& \overset{(c)}{>} -\widetilde{\Lambda}_0(u_0^*)-\log{2} \nonumber \\
& \overset{(d)}{=} E_{\mathrm{FP}, T_0}(\log{\frac{\pi_0}{\pi_1}})-\log{2\pi_0}.
\end{align}
Here (a) holds from \eqref{eq:uneq0}, (b) holds for any $u_0^*>0$, (c) holds because $u_0\tau'>0$, and (d) holds if we set $u_0^*=\arg\min \widetilde{\Lambda}_0(u)$ since $\widetilde{\Lambda}_0(u)$ attains its minima in $(0,1)$.

\textbf{Sub-case 1a:} $\tau'\geq \frac{d \widetilde{\Lambda}_1(u)}{d u}|_{u=0}= \mathrm{D}(P_1||P_0)-\log{\frac{\pi_0}{\pi_1}}$
\begin{align}
E_{\mathrm{FN}, T_0}(\tau_0)- \log{2\pi_1}=\sup_{u<0} (u\tau'-\widetilde{\Lambda}_1(u)-\log{2}) 
& \overset{(a)}{\leq} \sup_{u<0} (u \frac{d \widetilde{\Lambda}_1(u)}{d u}|_{u=0} -\widetilde{\Lambda}_1(u)-\log{2}) \nonumber \\
& \leq \sup_{u\in \mathcal{R}} (u \frac{d \widetilde{\Lambda}_1(u)}{d u}|_{u=0}-\widetilde{\Lambda}_1(u)-\log{2}) \nonumber \\
&\overset{(b)}{=} (0\frac{d \widetilde{\Lambda}_1(u)}{d u}|_{u=0}-\widetilde{\Lambda}_1(0)-\log{2}) \nonumber \\
& = (-\widetilde{\Lambda}_1(0)-\log{2}) \nonumber \\
& \overset{(c)}{<} -\min_{u}\widetilde{\Lambda}_1(u)-\log{2} \nonumber \\
& \overset{(d)}{=} E_{\mathrm{FP}, T_0}(\log{\frac{\pi_0}{\pi_1}})-\log{2\pi_0},
\end{align}
where (a) holds because $\tau' \geq \frac{d \widetilde{\Lambda}_1(u)}{d u}|_{u=0}$, (b) holds from \Cref{lem:supp_line}, (c) holds from the strict convexity of $\widetilde{\Lambda}_1(u)$ because it attains its minima in $(-1,0)$, and (d) holds from \eqref{eq:equal_exp}.

\textbf{Sub-case 1b:} $0 < \tau' <  \frac{d \widetilde{\Lambda}_1(u)}{du}|_{u=0}$
\begin{align}
 E_{\mathrm{FN}, T_0}(\tau_0) -\log{2\pi_0} = \sup_{u<0} (u\tau'-\widetilde{\Lambda}_1(u) - \log{2}) 
&\leq \sup_{u\in \mathcal{R}} (u\tau'-\widetilde{\Lambda}_1(u)- \log{2}) \nonumber \\
&\overset{(a)}{=}u_a\tau' -\widetilde{\Lambda}_1(u_a)- \log{2} \nonumber \\
& \overset{(b)}{<} -\widetilde{\Lambda}_1(u_a) - \log{2}\ \ \ \ \ [\text{since  } u_a\tau' < 0] \nonumber \\
& \leq -\min_{u}\Lambda_1(u) - \log{2}
\nonumber \\
& \overset{(c)}{=}E_{\mathrm{FP}, T_0}(\log{\frac{\pi_0}{\pi_1}})-\log{2\pi_0}
\end{align}
Here, (a) holds from \Cref{lem:supp_line} because $\widetilde{\Lambda}_1(u)$ is a strictly convex and differentiable function, and its derivative is also continuous, monotonically increasing and takes all values from $-\infty$ to $\infty$ (see \Cref{lem:unequal_log_gen}). Thus, there exists a single $u_a$ such that $\frac{d \widetilde{\Lambda}_1(u)}{du}|_{u=u_a}= \tau'.$
Next, (b) holds because $\frac{d \widetilde{\Lambda}_1(u)}{du}|_{u=u_a}= \tau' < \frac{d \widetilde{\Lambda}_1(u)}{du}|_{u=0}$, and the derivative is monotonically increasing, implying $u_a<0$. Lastly (c) holds from \eqref{eq:equal_exp}.

Thus,
\begin{align}
 E_{\mathrm{FN}, T_0}(\tau_0)-\log{2\pi_1}  < E_{\mathrm{FP}, T_0}(\log{\frac{\pi_0}{\pi_1}})-\log{2\pi_0}
< E_{\mathrm{FP}, T_0}(\tau_0)-\log{2\pi_0}.
\end{align}

For $\tau'<0$, a similar proof holds, leading to 
\begin{align}
 E_{\mathrm{FP}, T_0}(\tau_0)-\log{2\pi_0}  < E_{\mathrm{FP}, T_0}(\log{\frac{\pi_0}{\pi_1}})-\log{2\pi_0}
< E_{\mathrm{FN}, T_0}(\tau_0)-\log{2\pi_1} ,
\end{align}

Then, the value of $\tau_0$ that maximizes the Chernoff exponent $E_{e,T_0}(\tau_0)$, i.e., $$\max_{\tau_0}\ \min \{  E_{\mathrm{FP}, T_0}(\tau_0)-\log{2\pi_0}, E_{\mathrm{FN}, T_0}(\tau_0)-\log{2\pi_1} \},$$ is given by $\tau_0^*=\log{\frac{\pi_0}{\pi_1}}$ ($\tau'=0$). 

This matches with the detector that minimizes the Bayesian probability of error under unequal priors (see Theorem 3.1 in \cite{gallager2012detection}).

\textbf{Case 2:} $\frac{d \widetilde{\Lambda}_1(u)}{d u}|_{u=0}= \mathrm{D}(P_1||P_0)-\log{\frac{\pi_0}{\pi_1}}\leq 0.$

For this case, note that, both $\widetilde{\Lambda}_1(u)$ and $\widetilde{\Lambda}_0(u)$ attain their minima in $u \in [0,\infty)$.

For $\tau'= 0$ (equivalently $\tau_0=\log{\frac{\pi_0}{\pi_1}}$), we have
\begin{align}
 E_{\mathrm{FN}, T_0}(\log{\frac{\pi_0}{\pi_1}})-\log{2\pi_1}  = \sup_{u<0} (u\cdot0 -\widetilde{\Lambda}_1(u)-\log{2}) 
 =  -\widetilde{\Lambda}_1(0) -\log{2}. 
\end{align}
And,
\begin{align}
 E_{\mathrm{FP}, T_0}(\log{\frac{\pi_0}{\pi_1}})-\log{2\pi_0}= 
 \sup_{u>0} (u\cdot0 -\widetilde{\Lambda}_0(u)-\log{2}) 
& = -\min_{u} \widetilde{\Lambda}_0(u) -\log{2} \nonumber\\
& =  -\min_{u} \widetilde{\Lambda}_1(u) -\log{2} \nonumber \\
&  \geq -\widetilde{\Lambda}_1(0) -\log{2}. 
\end{align}
Thus,
\begin{align}
\min \{  E_{\mathrm{FP}, T_0}(\log{\frac{\pi_0}{\pi_1}})-\log{2\pi_0}, E_{\mathrm{FN}, T_0}(\log{\frac{\pi_0}{\pi_1}})-\log{2\pi_1} \} 
 = -\widetilde{\Lambda}_1(0) -\log{2}.
\end{align}

Now, we will show that any other value of $\tau'\neq 0$ (equivalently $\tau_0\neq \log{\frac{\pi_0}{\pi_1}}$) cannot increase the Chernoff exponent of the probability of error beyond $-\widetilde{\Lambda}_1(0) -\log{2}.$

\textbf{Sub-case 2a:} $\tau'\geq \frac{d \widetilde{\Lambda}_1(u)}{d u}|_{u=0}= \mathrm{D}(P_1||P_0)-\log{\frac{\pi_0}{\pi_1}}$
\begin{align}
E_{\mathrm{FN}, T_0}(\tau_0)- \log{2\pi_1}=\sup_{u<0} (u\tau'-\widetilde{\Lambda}_1(u)-\log{2}) 
& \overset{(a)}{\leq} \sup_{u<0} (u \frac{d \widetilde{\Lambda}_1(u)}{d u}|_{u=0} -\widetilde{\Lambda}_1(u)-\log{2}) \nonumber \\
& \leq \sup_{u\in \mathcal{R}} (u \frac{d \widetilde{\Lambda}_1(u)}{d u}|_{u=0}-\widetilde{\Lambda}_1(u)-\log{2}) \nonumber \\
&\overset{(b)}{=} (0\frac{d \widetilde{\Lambda}_1(u)}{d u}|_{u=0}-\widetilde{\Lambda}_1(0)-\log{2}) \nonumber \\
& = (-\widetilde{\Lambda}_1(0)-\log{2}),
\end{align}
where (a) holds because $\tau' \geq \frac{d \widetilde{\Lambda}_1(u)}{d u}|_{u=0}$ and (b) holds from \Cref{lem:supp_line}. Thus,
\begin{align}
\min \{  E_{\mathrm{FP}, T_0}(\tau_0)-\log{2\pi_0}, E_{\mathrm{FN}, T_0}(\tau_0)-\log{2\pi_1} \}  \leq -\widetilde{\Lambda}_1(0) -\log{2}.
\end{align}

\textbf{Sub-case 2b:} $\tau'< \frac{d \widetilde{\Lambda}_1(u)}{d u}|_{u=0}= \mathrm{D}(P_1||P_0)-\log{\frac{\pi_0}{\pi_1}}$
\begin{align}
 E_{\mathrm{FP}, T_0}(\tau_0)-\log{2\pi_0} = \sup_{u>0} (u\tau' -\widetilde{\Lambda}_0(u)-\log{2}) 
& \overset{(a)}{\leq} \sup_{u>0} (u \frac{d \widetilde{\Lambda}_1(u)}{d u}|_{u=0} -\widetilde{\Lambda}_0(u)-\log{2}) \nonumber \\
& \overset{(b)}{\leq} \sup_{u>0} (u \frac{d \widetilde{\Lambda}_0(u)}{d u}|_{u=1} -\widetilde{\Lambda}_0(u)-\log{2}) \nonumber \\
& \overset{(c)}{\leq} \sup_{u\in \mathcal{R}} (u \frac{d \widetilde{\Lambda}_0(u)}{d u}|_{u=1} -\widetilde{\Lambda}_0(u)-\log{2}) \nonumber \\
& \overset{(d)}{=} \frac{d \widetilde{\Lambda}_0(u)}{d u}|_{u=1} -\widetilde{\Lambda}_0(1)-\log{2} \nonumber \\
& \overset{(e)}{\leq}  -\widetilde{\Lambda}_0(1)-\log{2} \nonumber \\
& \overset{(f)}{=} -\widetilde{\Lambda}_1(0)-\log{2}.
\end{align}
Here (a) holds because $ \tau' < \frac{d \widetilde{\Lambda}_1(u)}{d u}|_{u=0}$, (b) holds from \Cref{lem:unequal_log_gen} since $\widetilde{\Lambda}_1(u)=\widetilde{\Lambda}_0(u+1)$, (c) holds because the supremum is taken over a larger superset, (d) holds from \Cref{lem:supp_line}, (e) holds because $\frac{d \widetilde{\Lambda}_0(u)}{d u}|_{u=1}= \frac{d \widetilde{\Lambda}_1(u)}{d u}|_{u=0} = \mathrm{D}(P_1||P_0)-\log{\frac{\pi_0}{\pi_1}} \leq 0$, and (f) holds again from  from \Cref{lem:unequal_log_gen} since $\widetilde{\Lambda}_1(u)=\widetilde{\Lambda}_0(u+1)$. Thus, 
\begin{align}
\max_{\tau_0} \min \{  E_{\mathrm{FP}, T_0}(\tau_0)-\log{2\pi_0}, E_{\mathrm{FN}, T_0}(\tau_0)-\log{2\pi_1} \} = -\widetilde{\Lambda}_1(0) -\log{2},
\end{align}
which is attained at $\tau_0=\log{\frac{\pi_0}{\pi_1}}$.

\end{proof}

\subsection{Unequal priors on both $Z$ and $Y$}
\label{app:unequal_z}

Here we discuss a modification of optimization \eqref{opt:both_group} proposed in \Cref{subsec:limit} to account for the case of unequal priors on both $Z$ and $Y$. 

Let $\Pr(Z=0)=\lambda_0$ and $\Pr(Z=1)=\lambda_1$. Also let, $\Pr(Y=0|Z=0)=\pi_{00}$, $\Pr(Y=1|Z=0)=\pi_{10}$, $\Pr(Y=0|Z=1)=\pi_{01}$ and $\Pr(Y=1|Z=1)=\pi_{11}$.

Then, the overall probability of error considering both groups together is given by:
\begin{align}
& \lambda_0 P_{e}^{T_0}(\tau_0)+ \lambda_1 P_{e}^{T_1}(\tau_1) \nonumber \\
& = \frac{1}{2}(2\lambda_0) P_{e}^{T_0}(\tau_0)+ \frac{1}{2}(2\lambda_1) P_{e}^{T_1}(\tau_1) \nonumber \\
& = \frac{1}{4}(4\lambda_0\pi_{00}) P_{\mathrm{FP},T_0}(\tau_0)+\frac{1}{4}(4\lambda_0\pi_{10}) P_{\mathrm{FN},T_0}(\tau_0) + \frac{1}{4}(4\lambda_1\pi_{01}) P_{\mathrm{FP},T_1}(\tau_1) + \frac{1}{4}(4\lambda_1\pi_{11}) P_{\mathrm{FN},T_1}(\tau_1).
 \end{align}

Then, the error exponent of the overall probability of error considering both groups is defined as:
\begin{align}
&\min\{ E_{\mathrm{FP},T_0}(\tau_0)-4\pi_{00}\lambda_0,E_{\mathrm{FN},T_0}(\tau_0)-4\pi_{10}\lambda_0,  E_{\mathrm{FP},T_1}(\tau_1)-4\pi_{01}\lambda_1,E_{\mathrm{FN},T_1}(\tau_1)-4\pi_{11}\lambda_1\}. 
\end{align}

These log-generating functions can be plotted, and the intercepts made by their tangents can be examined again to obtain the error exponents, leading to the optimal detector.

\end{document}